\documentclass{article} % For LaTeX2e
\usepackage{iclr2027_conference}

\usepackage{microtype}
\usepackage{graphicx}
\usepackage{subcaption}
\usepackage{booktabs} % for professional tables
\usepackage{caption}

\usepackage{amsmath}
\usepackage{amssymb}
\usepackage{mathtools}
\usepackage{amsthm}
\usepackage{listings}
\usepackage{multirow}

\newtheorem{theorem}{Theorem}[section]
\newtheorem{proposition}[theorem]{Proposition}
\newtheorem{lemma}[theorem]{Lemma}
\newtheorem{corollary}[theorem]{Corollary}

\newtheorem{definition}[theorem]{Definition}
\newtheorem{assumption}[theorem]{Assumption}

\newtheorem{remark}[theorem]{Remark}

\usepackage{amsmath,amsfonts,bm}

\def\eqref#1{equation~\ref{#1}}
\def\1{\bm{1}}

\DeclareMathAlphabet{\mathsfit}{\encodingdefault}{\sfdefault}{m}{sl}
\SetMathAlphabet{\mathsfit}{bold}{\encodingdefault}{\sfdefault}{bx}{n}

\usepackage{hyperref}
\usepackage{url}

\newcommand{\Real}{\mathbb R}

\newcommand{\cT}{\mathcal{T}}

\newcommand{\ex}{{\mathbb E}}

\newcommand{\ga}{\gamma}                %% abbreviated
\newcommand{\Th}{\Theta}                %%

\newcommand{\bs}{{\bf s}}               %%
\newcommand{\bu}{{\bf u}}               %% bold u
\newcommand{\bx}{{\bf x}}               %% bold x (vector)

\newcommand{\bp}{{\bf p}}
\newcommand{\bq}{{\bf q}}

\newcommand{\by}{{\bf y}}               %% bold y (vector)
\newcommand{\bz}{{\bf z}}               %% bold z (vector)

\newcommand{\bbZ}{{\bf Z}}

\newcommand{\diag}{\mbox{\sf Diag}}
\title{Constrained Triple-Head Price Optimization (C3PO) Network for Discrete-Choice Models}

\author{Shivaram Subramanian, Zhengliang Xue, Yingdong Lu, Markus Ettl, Jayant Kalagnanam
	\\ \\ 
		IBM Research
}

\iclrfinalcopy % Uncomment for camera-ready version, but NOT for submission.
\begin{document}

\maketitle

\begin{abstract}
We propose a Constrained Triple-Head Price Optimization (C3PO) network for real-time bi-level decision making in discrete-choice environments. In these settings, a seller selects product or service offerings to achieve business objectives, while heterogeneous users make personalized acceptance or rejection decisions based on their preferences, ultimately determining seller outcomes. Existing approaches that fine-tune pre-trained tabular foundation models on synthetic data are limited by their imitation-learning paradigm and inability to enforce common pricing rules, including cross-product and global constraints.
C3PO addresses these limitations through a novel transformer-based architecture that jointly integrates price imitation learning, multi-task revenue prediction, and in-context elasticity learning to generate robust pricing recommendations across domains without retraining. A lightweight differentiable constraint layer enables generated prices to satisfy business constraints while optimizing the revenue objective. During inference, frontier-model prompting elicits elasticity priors for previously unseen products informed by behavioral economics literature, improving out-of-distribution pricing effectiveness.
We demonstrate strong out-of-distribution performance on simulated, synthetic, and real-world datasets. C3PO is trained on data generated from multiple classical discrete-choice models, comprising simulated customer segments and counterfactual action-outcome pairs, and is evaluated on randomly generated choice environments without observing the underlying preference structure. Across simulated benchmarks where the calibrated MNL model is misspecified relative to the ground-truth choice model, C3PO improves expected revenue by 1.4 to 7.9 percentage points over a calibrated MNL baseline. The constraint layer reduces average absolute constraint violations by more than 50\%. We further deploy the model in real-world applications, including healthcare tender pricing and airline ancillary pricing, where C3PO improves directional pricing effectiveness, measured by price-increase and price-decrease recall, by 18--34 percentage points depending on the baselines considered, including calibrated MNL, classical optimization, and off-the-shelf tabular foundation-models.
\end{abstract}

\section{Introduction}

Discrete choice models (DCMs) are widely used in forecasting, assortment planning, pricing, marketing and other quantitative decision making processes across various industries such as transportation, restaurant and food service, retail, and entertainment when customers select among a finite set of products or services designed, provided and priced by the firms.
DCMs naturally capture a bi-level decision process well suited to competitive settings between a firm's offering and  customers' selection. From buyers' perspective, individual consumers evaluate the available options and select (or decline) items based on their own utility, driven by factors such as willingness to pay, perceived quality, and personal preferences. Meanwhile, the seller strategically sets the adjustable attributes of the offer set (for example, prices) to influence customer choices and optimize key performance metrics such as expected revenue or profit margin.
The decision-making literature has traditionally emphasized domain-specific models tailored to particular applications, rather than pursuing a universal model capable of capturing the essential relationship and common structure underlying diverse tasks and general decision strategies. This reliance on specialized models arises in part from the fact that real‑world decisions are shaped by complex human behavior and tacit knowledge that are difficult to capture with a single predictor or even an ensemble of traditional modeling approaches.

This paper aims to overcome these limitations by introducing a method for generating optimal actions that adapt to product types, customer heterogeneity and diverse constraints. Standard approaches optimize prices after fitting a parametric choice model (e.g., MNL) to each new domain. When the parametric family is misspecified, the choice model and resulting prices degrade, and the entire prediction-optimization pipeline must be rebuilt for every new product or market. An alternative is to fine-tune pre-trained tabular foundation models such as TabPFN~\citep{hollmann2022tabpfn} on pricing data, which mitigates the re-estimation cost but relies solely on imitation learning of unconstrained optimal actions and cannot effectively enforce the cross-product ordering, bound, and budget constraints that real pricing problems require. We design a new transformer-based architecture trained from scratch on simulated data spanning diverse discrete-choice families, that addresses both limitations. A \textbf{multitask component} captures the objective function (revenue--price curve) and its variation, shaping an \textbf{imitation learning module} that learns a representation of optimal prices, while an
\textbf{in-context learning module} guides predictions by modeling price elasticity rather than directly tracking the gradients of the objective. A lightweight differentiable constraint layer minimizes business-rule violations without post-hoc repair.

Whereas the focus of this paper is on pricing as an exemplar of discrete-choice based decision-making, the proposed approach is applicable beyond pricing to many other problem areas such as optimal product design, assortment optimization, transportation planning, and marketing, among others.

We first validate our approach in controlled settings using synthetic datasets generated from well-established ground-truth discrete choice models. Next, we apply the fine-tuned model to approximately solve a complex stochastic optimization problem arising in multiple real-world settings such as competitive tender pricing for medical device supply, and airline ancillary pricing, illustrating how the model can enhance pricing strategies for major industry participants. Finally, we benchmark the fine-tuned model recommendations against those obtained from traditional approaches and highlight the relative strengths and limitations of each approach.

\section{Related Work}
\paragraph{Discrete Choice Models.}
\citet{mcfadden1974conditional,mcfadden1974measurement}
introduced the random utility framework and conditional logit model, thus laying the theoretical foundation of discrete choice modeling for which D.  
McFadden was awarded the Nobel Prize in Economics in 2000. This fundamental theory is enriched by~\citep{anderson1992discrete} where the discrete choice theory of product differentiation was developed, and it provides a unified framework for modeling horizontal and vertical product differentiation that addresses the independence of irrelevant alternatives (IIA) problem. \citep{gallego2014multiproduct} derived the optimality conditions for maximizing expected margin for multinomial and nested logit models, which can be leveraged to calculate the optimal prices using a fixed point search. The work by \citep{zhang2018multiproduct} extended this result to pricing under generalized extreme value models (GEV). Data driven and nonparametric approach for decision-making in discrete choice models are pursued in~\citep{farias2009data,farias2013nonparametric}, and extended to Markov setting in~\citep{blanchet2016markov}.

Discrete choice models underpin a broad range of practical decision-making applications, including air travel demand forecasting and revenue management, ~\citep{garrow2016discrete}, personalized bundle pricing, ~\citep{xue2016pricing}, joint pricing and customer segmentation under business constraints via model‑free column generation, \citep{subramanian2022constrained}, and for multi-product revenue optimization across product lines, \citep{gallego2024bounds}. 
Related efforts in the vehicle routing domain include the work of \citep{mo2023predicting}  where the authors train a pair‑wise attention pointer network to infer real‑world driver routing behavior. Among other things, this network captures tacit driving hacks that are not readily modeled in classical vehicle routing formulations. 

\paragraph{Code Generation Models for Policy Learning.}
Recent models such as CodeT5 \citep{wang2021codet5} and StarCoder \citep{li2023starcoder} have been cited as having the ability to translate complex business constraints and  objectives into executable policy functions, exploiting their inherent understanding of conditional logic and structured decision rules essential for discrete choice policy generation.

\paragraph{Tabular Foundation Models.}
TabPFN \citep{hollmann2022tabpfn,hollmann2025accurate}, building on the PFN framework \citep{muller2021transformers}, demonstrates that transformers pre-trained on synthetic data can achieve state-of-the-art (SOTA) performance on tabular classification tasks. 
%However, these models focus exclusively on prediction. 

\section{Problem Formulation}
\label{sec:formulation}

Consider a marketplace where a product/service provider (seller) offers multiple ($K$) options to buyers who can select one of these choices to maximize their own utilities. Let $\bx, \by, \bz_i$ denote the product feature vector, customers' choice vector, and $i$-th customer attribute vector, respectively. Let $\bbZ=(\bz_1, \bz_2, \ldots, \bz_I)$ be the attribute vector for all customers and $\bp$ denote the pricing decision variable. Details on these variables and assumptions can be found in Appendix~\ref{sec:notations}.
Given choices, a customer's utilities are denoted by 
$\bu_i = (u_{i,1}, \ldots, u_{i,K})$, with corresponding choice probabilities 
$\bq = (q_{1}, \ldots, q_{K})$ based on the distribution of customer preferences. For example, the Multinomial Logit (MNL) model assumes that the random components of utility %($\boldsymbol{\epsilon}_i$) 
are independently and identically distributed (i.i.d.) according to a Gumbel 
distribution. Moreover, the default choice $0$ can be added, which is often interpreted as the baseline option, for example, it could refer to the selecting economy seats on a commercial flight relative to business, premium, or first class; or the no-purchase option, typically not having observed data for price or product features. 

Given these variables and their data instances, indexed by $n$, a multi-product pricing problem can be formulated as maximizing the expected revenue, subject to feasible prices in an action space $\Pi$:
\begin{equation}\label{eq:Max_rev}
    \max_{\bp_n \in \Pi} \; \bp_n \cdot \bq^{\top}(\bp_n, \bx_n, \bbZ_n).
\end{equation}
\vskip -0.2cm

Here, the probability of discrete choices $\bq(\bu_i)$ depends on the utility 
$\bu_i(\bp_n, \bx_n, \bbZ_n)$,  a function of the decision variables, 
non-decision features, and customer attributes. Consequently, the seller's optimal decisions depend on 
the non-decision features, for example through a pricing policy 
$\bp^{*}_n(\bx_n,\bbZ_n)$.

One viable approach to combining discrete choice models with machine learning techniques is to approximate $\mathbf{u}(\bx, \bz)$ with a neural network, and then obtain choice probabilities through a softmax operation, as shown in~\cite{doi:10.1287/mnsc.2023.02189}. Specifically, in~\cite{doi:10.1287/mnsc.2023.02189},  $\mathbf{u}(\bx, \bz)$ is treated as a random field and approximated via a system of neural networks accordingly. Their key technical result shows that, for any discrete choice model, there exists a neural network approximation whose induced choice probabilities are arbitrarily close (in total variation distance) to the original ones for any instance. The goal of our price optimization network is more ambitious: the model should not depend on any particular underlying choice model. Instead, a desirable theoretical guarantee is that, for any discrete choice model, the probabilities it generates can be approximated by those produced by the model, or equivalently, that the induced utility functions are themselves universally approximable; see~\citep{yuan2023power}.
Our model is then utilized to solve bi-level optimization~\eqref{eq:Max_rev} with the utilities, hence the probabilities, minimizes the loss function of the component for prediction. 
%Given customer features $\bz$ which optimize customers' utilities and product features $(\bp, \bx)$, the decision variables (e.g., prices) are optimized with respect to an objective function as in (\ref{eq:Max_rev}). 
While we utilize the pricing domain for clarity of exposition, the proposed framework is domain-agnostic and readily extensible to diverse decision-making tasks and action spaces. %such as revenue maximization, profit maximization, operational cost minimization, or assortment planning.
% \begin{align*}
%     \max f(\bp, \bq^*(\bp,\bx)), \ s.t.\ \bq^*(\bp,\bx)=&\arg\max_q \textcolor{blue}{ f(\bp, \bq)} 
%     \\ & \textcolor{blue}{\hbox{should this be a different function?} }.
% \end{align*}

\section{C3PO Model for Price optimization under Discrete Choice}
%Next, we examine modeling approaches for optimizing prices in discrete choice settings.

\subsection{Model Architecture}
\label{section::label_arch}

Figure~\ref{fig:architecture} illustrates the C3PO architecture, which maps customer attributes and historical price--revenue contexts to constrained price predictions via a hierarchical transformer with in-context learning. The model receives two primary inputs: a vector of customer features and a set of context samples $(p_i,r_i)$, where $p_i$ is a historical price vector and $r_i$ the corresponding revenue. At inference time, an elasticity signal---estimated from the context data and optionally anchored by a prior from economic literature---is incorporated as a control input.

\vskip -0.13cm
\begin{figure}[htbp!]
    \centering
    \includegraphics[width=0.75\linewidth]{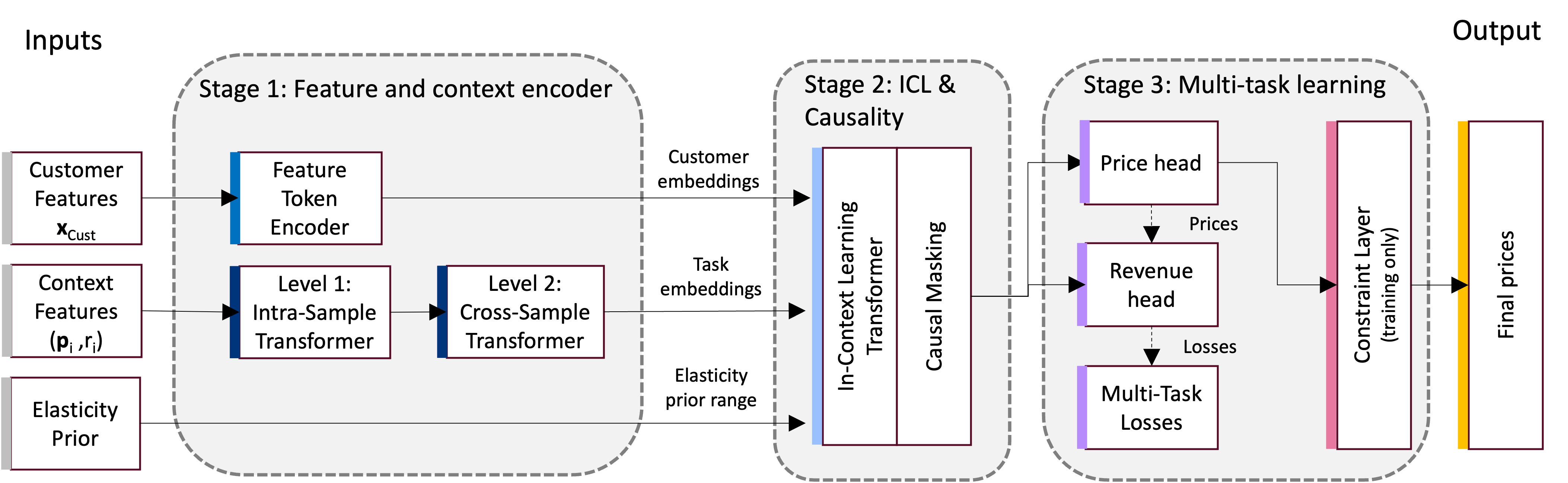}  
    \caption{High-level C3PO model architecture.}
    \label{fig:architecture}
\end{figure}

C3PO combines three learning paradigms:
(i)~\textbf{Imitation learning}---the price head imitates an oracle pricing policy via regression on labeled optimal prices;
(ii)~\textbf{Multi-task revenue modeling}---an auxiliary revenue head learns the price-to-revenue mapping $f(p)$, is frozen at inference, and serves as a differentiable reward signal directing the price head toward higher-revenue regions during training; and
(iii)~\textbf{In-context learning with elasticity priors}---labeled samples prepended to the batch enable few-shot prediction, while a product-type-specific elasticity range anchors prices. The model estimates per-customer elasticity from the context by comparing average revenue at below-median versus above-median price points; an external prior from the economics literature regularizes this estimate. Learning elasticity rather than the gradient $f'(p)$ improves cross-domain generalization and allows decision-makers to specify priors for new products via an LLM prompt.

\noindent
{\em Stage 1: Feature and Context Encoding.}\;
Customer features are mapped to embeddings by a feature token encoder. Each context sample $(p_i, r_i)$ is processed by an \emph{intra-sample} transformer that captures within-pair structure; the resulting embeddings are then aggregated by a \emph{cross-sample} transformer attending across all context samples, yielding a task embedding of the pricing environment.

\noindent
{\em Stage 2: In-Context Learning with Causal Masking.}\;
Customer and task embeddings are concatenated and passed to an ICL transformer with a causal mask, enabling few-shot conditioning on example rows without retraining. At inference, the elasticity estimate introduces an analytical anchor price derived from the Lerner optimality condition ( \cite{anderson1992discrete}) and a directional bias. The shared ICL representation feeds the price head.

\noindent
{\em Stage 3: Prediction Heads and Constraint Projection.}\;
The price head outputs candidate prices via imitation learning; the revenue head (active only during training) learns the price-to-revenue map, enabling gradient-based revenue maximization through the price head with the revenue head frozen. A differentiable constraint layer then projects raw prices onto a feasible set defined by box bounds, inter-product price ordering, and a global average-price constraint. The layer operates in three sequential passes: (1)~clamping each price to its box bounds, (2)~enforcing pairwise ordering by propagating adjustments along the product hierarchy, and (3)~redistributing any residual average-price violation proportionally across products, keeping the per-sample cost linear in $K$ (see Appendix~\ref{sec:constraintTH}). Constraint satisfaction is additionally encouraged via soft penalties during training. More rigorous alternatives such as OptNet~\citep{amos2017optnet} can substitute for this lightweight heuristic.

\paragraph{Training Objective.}
C3PO is trained with a multi-task loss:
\begin{equation}
\mathcal{L} = \mathcal{L}_{\text{price}} + \lambda_{\text{rev}}\mathcal{L}_{\text{rev}} + \lambda_{\text{revmax}}(t)\,\mathcal{L}_{\text{revmax}} + \lambda_{\text{constr}}\mathcal{L}_{\text{constr}} + \mathcal{L}_{\text{reg}}
\label{eq:total_loss}
\end{equation}
$\mathcal{L}_{\text{price}}$ is a Huber loss matching predicted prices to oracle optima.
$\mathcal{L}_{\text{rev}}$ calibrates the revenue head at observed context prices and ranks optimal prices above random alternatives via a margin loss.
$\mathcal{L}_{\text{revmax}}$ maximizes predicted revenue by backpropagating through the price head with the revenue head frozen; it is ramped in after a warmup once revenue predictions are sufficiently calibrated.
$\mathcal{L}_{\text{constr}}$ penalizes constraint violations on pre-projection prices.
$\mathcal{L}_{\text{reg}}$ comprises elasticity-prior alignment and analytical-anchor alignment losses (Appendix~\ref{appendix::label_training_loss}).

\subsection{Transfer Learning in C3PO}

The in-context learning (ICL) capability of C3PO learns to make optimal decisions from sequences of the form $(\mathbf{x}',\, \mathbf{q}(\mathbf{x}'),\, \mathbf{p}^*(\mathbf{x}'))$, for any $\mathbf{x}'$ in the testing dataset $\mathbf{X}'$ of the size $(M + 2K) \times N'$. For this type of problem, it is common to normalize the decision variables using an appropriate baseline. For details, see Appendix~\ref{sec:normalization}. There are two possible scenarios, depending on whether the choice probabilities $\mathbf{q}'$ are retrained during transfer learning.

\subsubsection{The First Scenario}

In the first scenario, no discrete-choice observations are available in the testing dataset. In this case, transfer learning relies on statistically contextual similarity between the pre-trained data and the testing data. The transformer then leverages the pre-trained optimal policy $\mathbf{p}^{\circ}(\mathbf{x}_n)$ learned from the training dataset $\mathbf{X}_{M \times N}$, represented by learned parameters in the form of matrix weights, to approximate the optimal policy $\mathbf{p}^*(\mathbf{x}')$ for any out-of-distribution $\mathbf{x}'$ that might be best quantified by a set of different learned parameters. The quality of this approximation depends on the degree of similarity between the distributions of the testing and training datasets, and is an intensively studied subject in \emph{domain adaptation}, see e.g.~\cite{redko2019advances}, a topic of transfer learning with different source and target distributions but the same task. Here, we adopt some fundamental ideas and methodologies in domain adaptation, and derive some approximation error estimations specific to the problem of optimal decision making for discrete choices models.

Suppose that the pre-trained data $\mathbf{x}_n$ follows a distribution over the feature space, which is usually termed \emph{source} distribution and denoted as $\pi_S$; the testing data $\mathbf{x}'$ follows a different distribution, is termed the \emph{target} distribution and denoted as $\pi_T$. Each statistical hypothesis is represented by a high dimensional matrix weight $\theta$, which means that the model $\cT_\theta$ takes an input feature vector $\bx$ to the output spaces (discrete for choice probabilities and continuous for optimal pricing). Specifically, here the model produces the optimal price as the solutions to the problem defined by~\eqref{eq:Max_rev}. The set of all statistical hypotheses is denoted as $\Th$. The error of learning is described by $\ex_\pi[\ell(\bx, \by;\theta)]$ where $\ell$ denotes the loss function characterizing the accuracy of the hypothesis with parameter $\theta$. 
\begin{assumption}
\label{asm:likelihoodRatio}
$\pi_T$ is absolute continuous with respect to $\pi_S$, and there exists a positive number $\delta>0$, such that $L_{S,T}$, the likelihood ratio of $\pi_T$ with respect to $\pi_S$, is upper bounded by $1+\delta$.
\end{assumption}
\begin{remark}
Depending on the topological structure of the feature space, the likelihood can be represented as either the ratio of probabilities(discrete case), or Radon-Nikodym derivative(continuous case), for details, see, e.g.,~\cite{rohde2014introductory}. 
\end{remark}
\begin{definition}
For a subset $A\subseteq \Th$, the $A$-divergence between two probability measures $\pi, \gamma$, denoted as $D_A(\pi, \gamma)$, is defined as $\sup_{a\in A} |\ex_\pi[\ell(\bx, \by;a)]-  \ex_\gamma[\ell(\bx,\by;a)]|$.
\end{definition}

%The following theorem characterizes the relationship of the errors, whose proof can be found in Appendix~\ref{sec:Thm_proof}.
The following theorem characterizes the error relationship (proof in Appendix~\ref{sec:Thm_proof}).
\begin{theorem}
\label{thm:domain_adaptation}
Suppose that the source and target distributions satisfy Assumption~\ref{asm:likelihoodRatio}, for $U$, a neighborhood of optimal matrix weight $\theta^*$ from pre-training, we have,  
\begin{align}
e_T(\theta^*)  \le & \min\{e_S(\theta^*)+\ga_U+D_U(\pi_S, \pi_T),(1+\delta)e_S(\theta^*)\}
\end{align}
with $e_T(\theta)=\ex_{\pi_T}[\ell(\bx, \by;\theta)]$, $e_S(\theta):=\ex_{\pi_S}[\ell(\bx, \by;\theta)]$, and $\ga_U:=\min_{\theta\in U}[e_T(\theta)+e_S(\theta)]$.
%\begin{align}
%e_T\le e_S+D_U(\pi_S, \pi_T)+R
%\end{align}
\end{theorem}
\begin{remark}
While Theorem~\ref{thm:domain_adaptation} falls into the family of typical results in domain adaptation, its novelty lies in its connection with normalized price range, and for the two-level LLM architecture, as well as the incorporation of the likelihood which is widely used in the pricing literature and introduces a multiplicative factor instead of the additive ones in the literature in domain adaptation. %Recall that given input space $X$ and label space $Y$ that is of finite cardinality, a domain $\cD$ is a distribution over $X\times Y$. For any predictor $h:X\to Y$, the \emph{error rate} is defined as $err_\cD: \pr_{(x,y)\sim \cD}(y\neq h(x))$. 

The definition of $A$-divergence is a variation of the $H\Delta H$ divergence that is introduced in~\cite{ben2010theory} and has been widely studied in domain adaption research. This definition allows us to specify a focus subset of hypotheses for the optimal pricing problem~\eqref{eq:Max_rev}. As we examine the potential hypotheses which are reflected as learned matrix weights in the network model, the analysis can be restricted to a neighborhood $U$ of the optimal weighted trained with the source data. Potentially, such restriction can lead to better estimations of the generalization error by utilizing well-documented ``flat minimum'' phenomenon for stochastic gradient descent, see~\cite{keskar2017large,jiangfantastic,singh2025avoiding,dziugaite2018entropy}. More precisely, a proper selected neighborhood may produce a small $D_U$-divergence between the source and target distributions, thus smaller generalization error.
Furthermore, the bound of the likelihood ratio characterizes relations between sensitivities of features (including demand) to price in the form of their joint probability, thus can be viewed as a more general quantification of the demand price elasticity as discussed in Appendix~\ref{appendix::label_elas_basics} in the form of conditional probability. In this sense, Theorem~\ref{thm:domain_adaptation} produces a very specific domain adaptation theory for optimal decision making for discrete choice models.
\end{remark}

\subsubsection{The Second Scenario}

In the second scenario, the testing dataset contains either (i) choice probabilities already estimated by the training pipeline or (ii) a small number of discrete-choice observations.

\noindent 
{\bf Choice Probabilities: } The choice probability vector $\mathbf{q}'$ is estimated using $(\mathbf{x}', \mathbf{p}')$ (or normalized prices $\tilde{\mathbf{p}}'$) where $\mathbf{p}' = \mathbf{p}(\mathbf{x}')$ or $\tilde{\mathbf{p}}' = \tilde{\mathbf{p}}(\mathbf{x}')$ is generated using either a predicted price $\mathbf{p}^{\dagger}$ or a near-optimal price $\mathbf{p}^{\circ}$ obtained from the first scenario. Therefore, the error of estimation can be quantified by Theorem~\ref{thm:domain_adaptation} with the transformer-based prediction function maps the feature inputs to probability vector defined over the possible choices.

\noindent 
{\bf Optimal Pricing: }
In our model, both the contextual information and the likelihood of discrete choices are incorporated. Thus, the optimal policy $\mathbf{p}^*(\mathbf{x}', \mathbf{q}'(\mathbf{x}'))$ is inferred via transfer learning from the pretrained policy $\mathbf{p}^{\bullet}(\mathbf{x}_n,\, \mathbf{q}(\mathbf{x}_n))$.
%\begin{eqnarray}\label{eq:FMDM_through_features_probs}
%\mathbf{p}^*(\mathbf{x}',\, \mathbf{q}'(\mathbf{x}')) \sim 
%\mathbf{p}^{\bullet}(\mathbf{x}_n,\, \mathbf{q}(\mathbf{x}_n)).
%\end{eqnarray}
The incorporation of either choice probabilities in the training or a small number of 
%discrete choice 
observations enriches probabilistic characterizations of the target domain, and thus naturally narrows the difference between the source and target domains, hence, 
\begin{lemma}
Suppose $\tilde{\pi}_T$ is obtained from $\pi_T$ by conditioning on additional observed labels on a subset of the feature space, then $D_A(\pi_S, {\tilde \pi}_T)\le D_A(\pi_S, \pi_T)$.
\end{lemma}
\vskip -0.2cm
This thus will lead to the following corollary concluding an improved bound on error estimation. 
\begin{corollary}
\label{col:domain_adaptation}
The error bound in the target domain can be improved to $
e_T(\theta^*)  \le  \min\{e_S(\theta^*)+\ga_U+D_U(\pi_S, {\tilde \pi}_T),(1+\delta)e_S(\theta^*)\}$
\end{corollary}
\vskip -0.2cm
Hence, our model enables transfer learning through the features and choice probabilities obtained above
%%in~\eqref{eq:TranferLearning4Features_Probs}, 
where $\mathbf{q}'(\mathbf{x}')$ is computed using predicted or near-optimal prices. %as in~\eqref{eq:prob_by_features}. 
Consequently, a network model obviates the need for computationally expensive optimization of an NP-hard problem for every transaction, significantly reducing the decision-processing time. After observing the actual prices $\mathbf{p}'$ and the corresponding outcomes $\mathbf{y}'$, a more accurate probability estimation can be obtained to improve optimization when necessary. Another advantage of our network model is that it can be trained or fine-tuned using one of many well-known approaches for generating synthetic tabular data (see Appendix ~\ref{appendix:Synth_Data_Gen} for more details). We demonstrate that it is sufficient to train the C3PO network using simulated data generated from classical discrete choice models in order to generalize well across diverse domains, heterogeneous customer preferences, and product types.

\section{Experiments}

\subsection{Simulated Data Generation For Model Training}
\label{sec:simulated_data_gen}

Training data are simulated using several classical DCMs: MNL ($16\%$), Nested Logit ($16\%$), Mixed MNL ($27\%$), Iso-elastic ($27\%$), and Linear ($14\%$).
Optimal price labels for MNL and NL datasets are computed via fixed-point search using the optimality condition in~\citep{gallego2014multiproduct}, and nonlinear optimization for other models.
The simulations yield 13{,}000 datasets with over 1.5 million labeled examples: 12{,}000 for training, the remainder for ablation studies, reserving 1{,}600 examples (13 datasets) for in-context learning (ICL). Segment-specific choice parameters $(\boldsymbol{\alpha}, \boldsymbol{\beta})$ are randomly generated for each dataset. 
Complete details on generating price labels and what-if outcome data are in Appendix~\ref{appendix::label_gen} and~\ref{appendix::label_training_data_and_plots}. The resultant C3PO model has about 337,000 trainable parameters.

%Larger scale C3PO models can be built by simply adding more simulated data that consider additional discrete choice model families and increasing ...

\subsection{Experimental Setup}
\label{sec:exp_setup}

C3PO is evaluated on real and simulated data from six domains: B2B tender-pricing, airline ancillary pricing, Amazon DVD sales \citep{farias2013nonparametric}, Hotel room selection \citep{newman2014estimation}, Swiss Metro ridership \citep{bierlaire2001swissmetro}, and Yogurt \citep{jain1994random}.
Dataset details are in Table~\ref{tab:meta} (Appendix~\ref{appendix::label_training_data_and_plots}); elasticity priors are drawn from the literature \citep{babic2005microeconomics, mankiw1998principles} and domain knowledge (Appendix~\ref{appendix::label_elas}).
 Noting that the available ground truth differs across datasets, we organize evaluation into three groups (Table~\ref{tab:revenue-combined}). Details on setup and baselines are in Appendix~\ref{appendix:dcm_datasets}. 
 \par
\textbf{Remarks:} C3PO was trained on datasets with $K\in[2,6]$ 
simulated choices and is evaluated without retraining on the DVD 
dataset with $K{=}15$ choices (Table~\ref{tab:meta} has additional dataset details). The Swiss Metro dataset is generated using the TasteNet-MNL model \citep{han2022neural} that embeds a neural network within an MNL structure to capture taste heterogeneity.

\iffalse
\begin{table}[thbp]
\centering\small\setlength{\tabcolsep}{5pt}
\begin{tabular}{llll}
\toprule
\textbf{Group} & \textbf{Datasets} & \textbf{Ground-truth model} & \textbf{Revenue metric} \\
\midrule
A: sim.\ optimal & DVD, Yogurt        & Mixed MNL            & regret from optimal \\
                  & Swiss Metro        & TasteNet-MNL         & regret from optimal \\
B: sim.\ offered  & Hotel              & MNL  & lift over offered \\
C: live           & B2B, AIRLINE          & unknown (real market) & --- (directional only) \\
\bottomrule
\end{tabular}
\caption{Evaluation groups and the revenue metric appropriate to each. 
%Groups~A and~B additionally report PDR/PIR against their respective reference prices.
}
\label{tab:eval-groups}
\end{table}
\fi

\subsection{Revenue Analysis}
\label{sec:revenue}

\paragraph{Experimental Setup.}
For Groups~A and~B the ground-truth choice model enables exact expected-revenue computation for any price vector. We evaluate four price sources per instance:
(i)~\textbf{Optimal/Reference}---revenue-maximizing prices (Group~A) or sampled offered prices (Group~B);
(ii)~\textbf{C3PO}---our model's recommendations;
(iii)~\textbf{MNL}---a calibrated single-model baseline (below);
and (iv)~\textbf{C3PO-Dir}---a deployment proxy adjusting each reference price by $\pm5\%$ in C3PO's indicated direction, modeling the intended operating regime of small, frequent, tactical adjustments. All prices are capped at $\pm1$ historical standard deviation around the mean.

\paragraph{Baseline.} Given that the ground truth models contain an MNL kernel, we generate a strong baseline by pricing using an MNL model (\cite{gallego2014multiproduct}) calibrated on the 1{,}600 ICL examples of each domain by first minimizing revenue MSE and then refining via maximum likelihood estimation.
%by minimizing revenue MSE, then %
%and use the Gallego--Wang fixed-point condition \citep{gallego2014multiproduct} for pricing. 
This standard predict-and-optimize pipeline is in-distribution when the ground truth is itself MNL.
%---an important control for attributing C3PO's gains to cross-family generalization rather than a weak comparator.

\paragraph{Results.}
Table~\ref{tab:revenue-combined} reports the ground-truth normalized revenue, regret or lift, and win rates. All datasets in groups A and B are evaluated on $10{,}000$ instances, except Yogurt which has $800$ held-out examples derived from the original dataset~\citep{jain1994random}. We summarize the key findings:
\begin{table}[thbp]
\centering\small\setlength{\tabcolsep}{3pt}
\renewcommand{\arraystretch}{0.92}
\begin{tabular}{@{}lll cc c c@{}}
\toprule
& & & \multicolumn{2}{c}{\textbf{GT Norm.\ rev.\ [95\% CI]}} & & \\
\cmidrule(lr){4-5}
\textbf{Grp} & \textbf{Dataset} & \textbf{Method} & \textbf{Mean} & \textbf{Med.} & \textbf{Regret / Lift} & \textbf{GT model} \\
\midrule
\multicolumn{7}{@{}l}{\emph{A --- regret from optimal (lower is better)}}\\[1pt]
& \multirow{3}{*}{DVD}
 & C3PO & $0.993^{***}$\;{\scriptsize[.991,\,.996]} & 0.964 & $-0.7\%$  & \multirow{3}{*}{Mixed MNL} \\
&& C3PO-Dir  & $1.000^{***}$\;{\scriptsize[.999,\,1.00]} & 0.999 & $-0.0\%$  & \\
&& MNL  & $0.914$\;{\scriptsize[.912,\,.916]} & 0.980 & $-8.6\%$  & \\[1pt]
& \multirow{3}{*}{SWISS}
 & C3PO & $0.826^{***}$\;{\scriptsize[.822,\,.829]} & 0.833 & $-17.4\%$ & \multirow{3}{*}{TasteNet-MNL} \\
&& C3PO-Dir  & $0.962^{***}$\;{\scriptsize[.961,\,.962]} & 0.954 & $-3.8\%$  & \\
&& MNL  & $0.812$\;{\scriptsize[.810,\,.815]} & 0.814 & $-18.8\%$ & \\[1pt]
& \multirow{3}{*}{YOGURT}
 & C3PO & $0.824^{***}$\;{\scriptsize[.817,\,.831]} & 0.848 & $-17.6\%$ & \multirow{3}{*}{Mixed MNL} \\
&& C3PO-Dir  & $0.978^{***}$\;{\scriptsize[.978,\,.978]} & 0.978 & $-2.2\%$  & \\
&& MNL  & $0.786$\;{\scriptsize[.779,\,.794]} & 0.773 & $-21.4\%$ & \\[1pt]
\midrule
\multicolumn{7}{@{}l}{\emph{B --- lift over offered prices (higher is better)}}\\[1pt]
& \multirow{3}{*}{HOTEL}
 & C3PO & $1.070^{***}$\;{\scriptsize[1.069,\,1.071]} & 1.067 & $+7.0\%$  & \multirow{3}{*}{MNL} \\
&& C3PO-Dir  & $1.031^{***}$\;{\scriptsize[1.030,\,1.031]} & 1.031 & $+3.1\%$  & \\
&& MNL  & $1.122$\;{\scriptsize[1.121,\,1.124]} & 1.154 & $+12.2\%$ & \\
\midrule
\multicolumn{7}{@{}l}{\emph{C --- live data (directional only; see Tables~\ref{tab:tender-kpi},~\ref{tab:pdrpir-live})}}\\[1pt]
& B2B, AIRLINE & --- & \multicolumn{3}{c}{\emph{ground-truth revenue unavailable}} & unknown \\
\bottomrule
\end{tabular}
\caption{Evaluation summary with 95\% bootstrap CIs ($B{=}10{,}000$).
Group~A: regret from oracle optimum; Group~B: lift over offered prices;
Group~C: directional metrics only.
$^{***}\!p<0.001$ vs.\ calibrated MNL (Wilcoxon signed-rank).}
\label{tab:revenue-combined}
\end{table}
\begin{itemize}
\item \textbf{C3PO outperforms MNL on all misspecified domains.} 
On DVD (mixed-MNL ground truth), C3PO's $-0.7\%$ regret dominates 
MNL's $-8.6\%$ for an improvement of $7.9$\,pp. On Yogurt (mixed-MNL), C3PO ($-17.6\%$) outperforms MNL ($-21.4\%$) by 
$3.8$\,pp, and on Swiss Metro (TasteNet-MNL), C3PO ($-17.4\%$) 
outperforms MNL ($-18.8\%$) by $1.4$\,pp. All differences are 
significant at $p<0.001$ (Wilcoxon signed-rank). C3PO's gains 
are largest when the degree of MNL misspecification is greatest.

\item \textbf{C3PO-Dir confirms deployment potential.} The C3PO-Dir directional proxy adjusts reference prices by a small fixed step in C3PO's recommended direction. C3PO-Dir is ground-truth aware when the reference is optimized (Group~A) and is the dominant Group~A performer. This result suggests that when C3PO's direction is correct, it can apply small repeated adjustments to a good-quality reference price and converge to near-optimal revenue.

\item \textbf{MNL wins in-distribution.} On Hotel (MNL ground truth, offered-price reference), MNL achieves the largest lift ($+12.2\%$) while C3PO still delivers a significant $+7.0\%$ lift over offered prices, remaining competitive across choice-model families.
\end{itemize}

\paragraph{Statistical testing.}
All pairwise revenue comparisons use the Wilcoxon signed-rank test on
per-instance differences; 95\% bootstrap CIs ($B{=}10{,}000$) are
reported in Table~\ref{tab:revenue-combined}. All C3PO-vs-MNL and
C3PO-Dir-vs-MNL comparisons are significant at $p<0.001$ across all
four datasets.
\begin{remark}
The domain adaptation bound in Theorem~\ref{thm:domain_adaptation} provides qualitative insight into C3PO's empirical behavior: the largest revenue gains over MNL occur on DVD and Yogurt, where the mixed-MNL ground truth is most distant from the MNL family, corresponding to larger $\ga_U$ for the MNL baseline but not for C3PO's nonparametric hypothesis class.
\end{remark}

\subsection{Directional Performance (PDR/PIR)}
\label{sec:PDRPIR}
Price Increase Recall (PIR) and Price Decrease Recall (PDR), as defined in \cite{ye2018customized}, measure how effectively a model identifies opportunities to raise prices on winning bids (PIR) and to lower prices on losing bids (PDR). A \emph{balanced} pair---high $\min(\text{PDR},\text{PIR})$---signals a well-calibrated recommender rather than one systematically biased toward raising or lowering prices. 

\textbf{Live datasets in group C}: The B2B ICL dataset consists of 1600 competitive tender bids sampled from four years of historical data having an average win rate of about 34\%, and is used to price bids for the next year. The reference prices are generated by the bidder's rule-based system. The airline ICL dataset consists of 1600 examples of upgrade purchases for a global carrier, and is used to price 10{,}000 upgrade choices across diverse test markets. The historical win rate is less than 14\% and the reference prices are optimized values generated by a sophisticated revenue management system.

\subsubsection{B2B: in-depth comparison.}
The B2B setting is the canonical buy vs no-buy problem and
admits the richest set of baselines. We compare C3PO against (off-the-shelf)
TabPFN, a fine-tuned TabPFN (TabPFN-FT), and the constrained stochastic-optimization
policy of \citet{subramanian2022constrained}, which uses full model knowledge
(Table~\ref{tab:tender-kpi}). The TabPFN-FT fine-tuning details are in Appendix \ref{appendix::tabpfn_ft}. C3PO attains the highest
$\min(\text{PDR},\text{PIR})=0.611$, exceeding the stochastic optimizer
($0.430$), TabPFN-FT ($0.589$), and TabPFN ($0.393$). TabPFN and the stochastic policy both exhibit a high PDR but a
weak PIR while C3PO instead produces the most
\emph{balanced} PDR-PIR, which is our directional KPI.

\subsubsection{Directional performance across all domains.}
Table~\ref{tab:pdrpir-live} compares the PDR, PIR, and $\min(\text{PDR},\text{PIR})$
for C3PO versus MNL for the live datasets. The PDR-PIR comparison for all datasets is available in Table~\ref{tab:pdrpir-all} of the appendix. A consolidated table of PDR-PIR and revenue gain for groups A and B is presented in appendix~\ref{sec:rev-pdr-correlation}.

C3PO outperforms MNL on directional accuracy across all domains, 
maintaining balanced PDR/PIR ($\min > 0.41$) on every dataset. The live datasets represent real choice processes involving unobserved competitive heterogeneity (B2B) and nested substitution with extreme class imbalance (airline), neither of which MNL can represent regardless of sample size. 
%---sensible for an inexpensive grocery item but producing an unbalanced pair and, as we show next, elevated revenue regret. 

\begin{table}[thbp]
\begin{minipage}{0.42\textwidth}
    \centering
    \captionsetup{font=tiny}
\tiny
\begin{tabular}{lccc}
\toprule
\textbf{B2B Pricing Model} & \textbf{PDR} & \textbf{PIR} \\
\midrule
C3PO                   & \textbf{0.611} & \textbf{0.673}\\
Baseline TabPFN        & 0.767 & 0.393\\
TabPFN-FT              & 0.723 & 0.589\\
Stochastic Opt.\ Policy & 0.646 & 0.430\\
MNL                  & 0.443 & 0.299\\
\bottomrule
\end{tabular}
\caption{B2B directional pricing metrics: $\min(\text{PDR},\text{PIR})$ is
maximized by C3PO. Stochastic-optimization and TabPFN both
under-price.}
\label{tab:tender-kpi}
 \end{minipage} \hspace{0.3cm}
 \begin{minipage}{0.45\textwidth}
    \centering
    \captionsetup{font=tiny}
\tiny
\begin{tabular}{l cc cc cc}
\toprule
& \multicolumn{2}{c}{\textbf{PDR} $\uparrow$}
& \multicolumn{2}{c}{\textbf{PIR} $\uparrow$}
& \multicolumn{2}{c}{$\mathbf{\min(\text{PDR},\text{PIR})}$ $\uparrow$} \\
\cmidrule(lr){2-3} \cmidrule(lr){4-5} \cmidrule(lr){6-7}
\textbf{Dataset}
& \textbf{C3PO} & \textbf{MNL}
& \textbf{C3PO} & \textbf{MNL}
& \textbf{C3PO} & \textbf{MNL} \\
\midrule
B2B      & \textbf{0.611} & 0.443 & \textbf{0.673} & 0.299 & \textbf{0.611} & 0.299 \\
AIRLINE  & \textbf{0.507} & 0.135 & \textbf{0.476} & 0.318 & \textbf{0.476} & 0.135 \\
\midrule
$\Delta$ (avg) & \multicolumn{2}{c}{$+0.270$} & \multicolumn{2}{c}{$+0.266$} & \multicolumn{2}{c}{$+0.327$} \\
\bottomrule
\end{tabular}
\caption{Directional pricing metrics on live datasets.
C3PO improves $\min(\text{PDR},\text{PIR})$ over MNL
by $+31$\,pp on B2B and $+34$\,pp on Airline.}
\label{tab:pdrpir-live}
 \end{minipage}
\end{table}

The calibrated MNL baseline systematically overprices by underestimating price elasticity from limited data, yielding low PDR across all domains (see Appendix~\ref{sec:detailed_pdr_pir_stats} for analysis). 
%MNL's failure mode on these datasets is not estimation noise but structural misspecification: 
On the other hand, C3PO's non-parametric approach maintains balanced directional accuracy by learning flexible, context-conditioned price mappings pre-trained across thousands of diverse choice environments.

\subsection{Ablation Studies}
The C3PO architecture is designed to overcome the limitation of TabPFN in three areas, which we want to analyze through ablation studies:
1. Move beyond imitation-learning of prices to incorporate multi-task learning of revenue prediction. 
2. The impact of eliminating or reducing ICL examples. 
3. Ability to learn the elasticity for a new domain via ICL examples and the use of an elasticity prior that is available in the literature.
The ablation studies analyze the impact of each of these components of C3PO. Simulated test data from the various discrete choice models used to train C3PO are included as an additional dataset in these studies. Table \ref{tab:avg_metrics} reports the results averaged across all the datasets. 
%\vskip -0.3cm
%\input{NeurIPS/ablation_results/average_metrics_by_scenario}

\begin{table}[ht]
\centering\scriptsize
\setlength{\tabcolsep}{2.5pt}
\begin{minipage}[t]{0.37\linewidth}
\centering
\begin{tabular}{lccc}
\toprule
Scenario & MAE & PDR & PIR \\
\midrule
C3PO           & .38$\pm$.11 & .49$\pm$.12 & .54$\pm$.13\\
ICL-OFF        & .66$\pm$.13 & .20$\pm$.16 & .81$\pm$.17\\
CONSTR-ON      & .31$\pm$.15 & .43$\pm$.30 & .57$\pm$.30\\
IMIT-ONLY      & .35$\pm$.06 & .59$\pm$.20 & .41$\pm$.16\\
PRIOR-OFF      & .41$\pm$.08 & .44$\pm$.15 & .58$\pm$.19\\
SIMPLE-ICL     & .40$\pm$.08 & .43$\pm$.14 & .59$\pm$.18\\
\bottomrule
\end{tabular}
\captionsetup{font=scriptsize}
\caption{Avg.\ pricing metrics (mean$\pm$std).}
\label{tab:avg_metrics}
\end{minipage}%
\hfill
\begin{minipage}[t]{0.61\linewidth}
\centering
\begin{tabular}{lcccc}
\toprule
 & \multicolumn{2}{c}{Abs.\ Mean} & \multicolumn{2}{c}{\% Mean} \\
\cmidrule(lr){2-3}\cmidrule(lr){4-5}
Constraint & Unconstr. & Constr. & Unconstr. & Constr. \\
\midrule
Box           & .19$\pm$.39 & .07$\pm$.16 & 9.7$\pm$19.7 & 3.3$\pm$8.2 \\
Ordering      & .15$\pm$.17 & .06$\pm$.09 & 16.9$\pm$20.2 & 9.8$\pm$15.5 \\
Avg.\ Price   & .07$\pm$.10 & .03$\pm$.04 & 6.0$\pm$8.1 & 3.3$\pm$5.5 \\
\bottomrule
\end{tabular}
\captionsetup{font=scriptsize}
\caption{Avg.\ constraint violations (mean$\pm$std).}
\label{tab:avg_constraints}
\end{minipage}
\end{table}

In the first study we turn off the attention network that models the revenue-price curve and replace it with a statistical summary. This reduces C3PO to an 'imitation-only' learning model akin to TabPFN.  
The performance deteriorates when it attempts to generalize to some of the new domains as can be seen in the skewed PDR-PIR numbers, and the min(PDR, PIR) drops from C3PO's score of 0.49 to 0.41. The detailed results of this simplified network are shown in Table \ref{tab:imitation_only} in Appendix \ref{appendix::label_ablation}.

The second study evaluates the impact of ICL on pricing. In one setting, ICL is disabled ("ICL‑OFF"). In another, the ICL examples are reduced to 100 ("SIMPLE‑ICL"). With no ICL, the model fails to generalize effectively, with recommended prices clustering near the upper bound and the KPI dropping to 0.2, and just 100 ICL examples suffice to improve the KPI to 0.43. In contrast, the MNL calibration quality is expected to degrade when the calibration data reduces from 1600 to 100 samples as it relies on point estimates of $\beta$. Tables \ref{tab:icl_off} and \ref{tab:simple_icl} in Appendix \ref{appendix::label_ablation}) present detailed results.

The third study introduces an explicit elasticity prior to account for the consumers' product-specific willingness to pay based on necessity and the availability of substitutes. The results are shown in Table \ref{tab:prior_off} in Appendix \ref{appendix::label_ablation}. The prior is  useful when a relatively narrow elasticity range can be derived using domain knowledge. In the absence of a product-specific prior, model accuracy declines and the KPI drops to 0.44 as the system falls back to the default elasticity range of (-3.0, -1.0). 

%The weights for the training loss terms reflect a decision-maker's preference for emphasizing the elasticity prior. 
%We recommend a (0.5–0.75) weight for the elasticity prior component in real-world settings where a reasonably narrow elasticity range can be derived using domain knowledge. 
%Such a prior beneficially shifts the optimal price distribution relative to the training regime and generally improves accuracy without requiring retraining. 

\subsection{Constraint Satisfaction}

We report on the performance of the heuristic constraint satisfaction layer described in Section \ref{section::label_arch} 
%that works adequately for a blackbox optimization approach in a discrete choice decision-making setting. 
Table \ref{tab:avg_constraints} tracks the constraint satisfaction percentage and average percentage violation across the datasets, for three common constraint types that show up in practical discrete choice settings.

Table  \ref{tab:constr_baseline_results} and \ref{tab:constr_on_final} in Appendix~\ref{appendix::label_training_data_and_plots} depict the feasibility performance before, and after, activating the constraint layer, respectively, for each of the datasets, We observe a substantial reduction in infeasibility for most domains. Table \ref{tab:avg_metrics} summarizes the directional KPI dropping from 0.49 to 0.43 after activating the constraint layer, while reducing the absolute constraint violations by more than $50\%$. A post-processing step can be applied to project the final predictions into the feasible space.

\section{Conclusions}
\label{sec:conclusions}

We introduce C3PO, a constrained triple-head transformer for bi-level price optimization in discrete-choice settings. By unifying imitation learning, multi-task revenue modeling, and in-context elasticity learning with a differentiable constraint layer, C3PO produces domain-adaptive pricing policies from a single model trained entirely on simulated data spanning diverse choice-model families.

On four benchmarks with known ground-truth models, C3PO improves expected revenue over calibrated MNL by 1.4--7.9\,pp when MNL is misspecified, with all gains significant at $p<0.001$. The directional proxy C3PO-Dir, which applies small repeated adjustments in C3PO's recommended direction, achieves near-optimal revenue on all Group~A datasets, confirming deployment viability. C3PO improves balanced directional accuracy by 31--34\,pp over calibrated MNL.

The architecture scales naturally: training on additional choice-model families and customer-segment scenarios requires only generating more simulated data, with no structural changes. Inference-time elasticity priors, obtainable via a single LLM prompt, further extend the model to unseen product categories without retraining. Limitations include degraded performance in extreme elasticity regimes and approximate constraint satisfaction. Future work will address these via richer training distributions and differentiable optimization layers.

\nocite{langley00}

%%%%%%%%%%%%%%%%%%%%%%%%%%%%%%%%%%%%%%%%%%%%%%%%%%%%%%%%%%%%%%%%%%%%%%%%%%%%%%%
%%%%%%%%%%%%%%%%%%%%%%%%%%%%%%%%%%%%%%%%%%%%%%%%%%%%%%%%%%%%%%%%%%%%%%%%%%%%%%%
% APPENDIX
%%%%%%%%%%%%%%%%%%%%%%%%%%%%%%%%%%%%%%%%%%%%%%%%%%%%%%%%%%%%%%%%%%%%%%%%%%%%%%%
%%%%%%%%%%%%%%%%%%%%%%%%%%%%%%%%%%%%%%%%%%%%%%%%%%%%%%%%%%%%%%%%%%%%%%%%%%%%%%%
\newpage
\appendix
\onecolumn

\section{Notations and Assumptions for the Optimization Formulation}
\label{sec:notations}

\begin{itemize}
\item 
There are $K$ products offered by the seller, indexed by $k = 1, \dots,K$. Accordingly, a customer's choice is characterized by $\by = (y_{1}, \ldots, y_{K}) \in \mathbb{R}^{K} $, identifying the purchased product $y$, satisfying $\sum_{k=1}^{K} y_{k} = 1$. In the case that default choice $0$ is added, we will have $\sum_{k=0}^{K} y_{k} = 1$.
\item
Seller's decisions are also characterized by a vector $\bp=(p_1, p_2, \dots, p_K) \in \mathbb{R}^{K}$, a representative example is a price vector, associated with the revenue vector ${\bf r}=(r_1,r_2, \ldots, r_K)$, which will be the one considered in this paper.
\item
There are $M$ observable, non-decision product features, measured either as real-valued variables or elements of finite lattices (e.g., binary indicators). These numeric and categorical attributes are denoted by $\bx = (x_{1}, \ldots, x_{M})\in \mathbb{R}^{M}$ for any purchased record.
\item
There are $L$ observable customer attributes, denoted by $\bz_i = (z_{i,1}, \ldots, z_{i,L})\in \mathbb{R}^{L}$ for each customer (or customer segment) $i = 1,2,\ldots,I$, and $\bbZ=(\bz_1, \bz_2, \ldots, \bz_I)$.
\item
There are $N$ historical transaction data indexed by $n = 1,\ldots,N$. Each transaction is represented by $(\bp_n, \bx_n, \by_n, \bbZ_n)$.
\end{itemize}

\section{Theoretical Development}
\label{sec:theory}

\subsection{Methodologies for Decision Model under Discrete Choice}

\subsubsection{Normalization of the Decision Variables}
\label{sec:normalization}

Normalization is important both for the optimization procedure and for obtaining meaningful insights into the resulting optimal decisions, and is conceptually distinct from the feature standardization used in prediction tasks.
For example, pricing models often require negative demand elasticity, which
places structure on how normalized prices should behave. The choice of baseline
varies across pricing applications and may include cost, list price,
competitive price, etc. 
For example, in an airline dataset, we normalize the upgraded-service prices using the base economy fare $p_{0,n}$ as the benchmark, yielding normalized prices of the form
$
\tilde{p}_{k,n} = \frac{p_{k,n}}{p_{0,n}} - 1,
$
for any $k$ in the upper classes. 

\subsubsection{Price Elasticity}
\label{appendix::label_elas_basics}
Some basic facts and results that can enhance our understanding of the basic mechanism between the optimal pricing decisions and the key quantity of price elasticity, thus improve the design and execution of C3PO, are collected below.

Some basic facts in optimization can help us to understand the basic mechanism between the optimal pricing decisions and the key quantity of price elasticity. 
\begin{definition}
The demand elasticity matrix
$E\in\Real^{K\times K}$ is defined as  $E_{jk}=\frac{p_k}{q_j}\frac{\partial q_j}{\partial p_k}$ for all $j, k=1,2,\ldots K$. 
\end{definition}
\begin{remark}
The definition of demand elasticity matrix, which can also be written as $E_{jk}=\frac{\partial q_j}{q_j}/\frac{\partial p_k}{p_k}=\frac{\partial \ln q_j}{\partial \ln p_k}$ and should be viewed as scaled gradient of the demand with respect to price, and it follows the basic logic of being the ratio between changes in demand and changes in price that is well known in the economics literature, see e.g.~\cite{parkin2008economics}. The diagonal entries $E_{kk}$ are \emph{own-price elasticities} and they are nonpositive naturally. The off-diagonal entries $E_{jk}$ ($j\neq k$) are \emph{cross-price elasticities}, and they are typically positive for substitutes. Comparing to the slope of demand curve, price elasticity is scale-invariant and dimensionless, thus is preferred for analysis of diverse product and with counterfactual actions such as ours.
\end{remark}

In the one product case, i.e. $K=1$, the price elasticity of demand takes the form of $E(p):=pq'(p)/q(p)$. Without any constraints, then it is easy to see from the optimal condition, we have $\epsilon(p^*)=-\frac{p^*}{p^*-c}$, in particular, $E(p^*)=-1$ when $c=0$. In the multi-product case, several basic relationships between the optimal price and the cost structure can be derived easily from the optimality condition.

\begin{proposition}[multi-product optimality condition]
\label{thm:multiproduct}
At an interior optimum $\bp^*$, the first-order conditions in vector form are
\begin{equation}
\diag(\bp)\bq/R + E(\bp^*)^\top\, \bs(\bp^*) = \mathbf{0}\,,
\label{eq:multiproduct_foc_mod}
\end{equation}
where $\bs=[s_1, s_2, \dots, s_K]^\top$, $s_k = (p_k-c_k)\,q_k / R, k=1,2,\ldots, K$ is the \emph{profit-weighted revenue share} of product $k$.
\end{proposition}

\begin{proof}
Setting $\partial R/\partial p_k=0$ for each $k=1,2,\ldots, K$:
\begin{equation}
q_k + \sum_{j=1}^{K}(p_j-c_j)\frac{\partial q_j}{\partial p_k} = 0\,.
\label{eq:foc_k}
\end{equation}
Thus, 
\begin{align*}
q_k + \sum_{j=1}^{K}\frac{(p_j-c_j)\,q_j}{p_k}\cdot\underbrace{\frac{p_k}{q_j}\frac{\partial q_j}{\partial p_k}}_{E_{jk}} = 0
\end{align*}
Equivalently, 
\begin{align*}
q_k+\sum_j \frac{(p_j-c_j)\,q_j}{R}\cdot\frac{R}{p_k}\cdot E_{jk}=0\,.
\end{align*}
This implies
\begin{align*}
\frac{p_kq_k}{R}+\sum_j \frac{(p_j-c_j)\,q_j}{R}\cdot E_{jk}=0\,.
\end{align*}
Equation~(\ref{eq:multiproduct_foc_mod}) thus follows. 
\end{proof}

\subsubsection{MNL Closed Form and the IIA Structure}

In MNL with utilities $V_k=\alpha_k-\beta p_k$, i.e., price sensitivity is common across all products and outside option $V_0=0$, we have, 
\begin{align*}
q_k(\bp) = \frac{e^{\alpha_k-\beta p_k}}{1+\sum_j e^{\alpha_j-\beta p_j}}\,.
\end{align*}
A direct consequence is that, 
\begin{proposition}[MNL elasticity structure]
\label{prop:mnl_elast}
\begin{align}
E_{kk} &= -\beta\,p_k\,(1-q_k)\,, \label{eq:mnl_own}\\
E_{jk} &= +\beta\,p_k\,q_k \quad (j\neq k)\,. \label{eq:mnl_cross}
\end{align}
\end{proposition}
\begin{proof}
$\partial q_k/\partial p_k = -\beta\,q_k(1-q_k)$ and $\partial q_j/\partial p_k = \beta\,q_j\,q_k$ for $j\neq k$ (standard MNL derivatives). Multiplying by $p_k/q_k$ and $p_k/q_j$ respectively.
\end{proof}
Furthermore, 
\begin{theorem}[MNL multi-product optimal price]
\label{thm:mnl_opt}
The revenue-maximizing price for product $k$ satisfies
\begin{equation}
p_k^* = \frac{1}{\beta(1-q_k^*)}\,.
\label{eq:mnl_pstar}
\end{equation}
%For symmetric products ($\alpha_k=\alpha\ \forall k$), $q_k^*=q^*$ for all $k$ and the unique symmetric solution is
%\begin{equation}
%p^* = \frac{1}{\beta(1-q^*)}\,,\qquad q^* = \frac{K\,e^{\alpha-\beta p^*}}{1+K\,e^{\alpha-\beta p^*}}\,.
%\label{eq:mnl_sym}
%\end{equation}
\end{theorem}

\begin{proof}
The FOC~\eqref{eq:foc_k} with $c_k=0$ gives $q_k + \sum_j p_j\,\partial q_j/\partial p_k = 0$. Using MNL derivatives:
\[
q_k + p_k(-\beta\,q_k(1-q_k)) + \sum_{j\neq k}p_j\,\beta\,q_j\,q_k = 0\,.
\]
Dividing by $q_k>0$: $1-\beta\,p_k(1-q_k)+\beta\sum_{j\neq k}p_j\,q_j=0$. $1-\beta\,p_k(1-q_k)+\beta\sum_{j\neq k}p_j\,q_j=0$ immediately implies $1-\beta p_k+\beta\sum_{j=1}^Kp_j\,q_j=0$, equivalently, $p^*_k=(1+\beta R)/\beta$. The fixed-point $p_k^* = (1+\beta R^*)/\beta$ combined with $R^* = \sum_k p_k^* q_k^*$ yields $p_k^* = 1/[\beta(1-q_k^*)]$ after algebraic rearrangement.
\end{proof}

\subsection{Bi-level Optimization}
As we elaborated in the introduction, the optimization problem~\eqref{eq:Max_rev}  belongs to the extensively studied domain of bi-level optimization. A general formulation of bi-level optimization takes the following form,
\begin{align}
\label{eqn:bilevelOpt}
%% \min_{\bx\in \Omega \subseteq \Real^{d_x}} \ell(\bx) := f(\bx, \by^*(\bx)), \quad \text{s.t. } \by^*(\bx) \in \arg\min_{\by\in \Real^{d_y}} g(\bx, \by).
\min_{\bx\in \Real^{d_x}} \ell(\bx) := f(\bx, \by^*(\bx)), \quad \text{s.t. } \by^*(\bx) \in \arg\min_{\by\in \Real^{d_y}} g(\bx, \by),
\end{align}
with given functions $f(\bx, \by)$ and $g(\bx,\by)$. Under strong (usually convexity) conditions on the functions, algorithms have been developed and shown to be efficient. For problems with complex inner level structures such as those in~\eqref{eq:Max_rev}, algorithmic development is still quite limited.
For theory, algorithms and other detailed discussion on bi-level optimization, see e.g.~\cite{borkar1997stochastic,zhang2024introduction}.

\subsection{Proof of Theorem~\ref{thm:domain_adaptation}}
\label{sec:Thm_proof}

\begin{proof}[Proof of Theorem~\ref{thm:domain_adaptation}]
In a parameterized learning regime, hypotheses are represented by the optimized parameter. In the case of transformers, they are the weights after the training, which are represented as $\theta^*$. Consider $U$, a neighborhood of $\theta^*$, for any $\theta'\in U$, define $e_T(\theta, \theta'):=\ex_{\pi_T}[|\ell(\bx, \by;\theta)-\ell(\bx, \by;\theta')|]$ and $e_S(\theta, \theta'):=\ex_{\pi_S}[|\ell(\bx, \by;\theta)-\ell(\bx, \by;\theta')|]$. We have
\begin{align*}
e_T(\theta^*) \le & e_T(\theta')+ e_T(\theta^*, \theta') \\ \le & e_T(\theta')+ e_S(\theta^*, \theta')+  |e_T(\theta^*, \theta')-e_S(\theta^*, \theta')|
\\ \le & e_S(\theta^*)+e_T(\theta')+e_S(\theta')+D_U(\pi_S, \pi_T),
\end{align*}
where the inequalities follow from the definition of $e_T(\theta, \theta')$ and the triangle inequality.
Therefore, we can conclude that 
\begin{align*}
e_T(\theta^*)  \le & e_S(\theta^*)+\ga_U+D_U(\pi_S, \pi_T).
\end{align*}
%with $\g_U:=\min_{\theta\in U}[e_T(\theta)+e_S(\theta)]$.
On the other hand, note that, $e_T(\theta^*) =\ex_{(x,y)\sim D_T}[\ell(\bx, \by;\theta^*)]$, and 
\begin{align*}
\ex_{(x,y)\sim D_T}[\ell(\bx, \by;\theta^*)]&=\ex_{(x,y)\sim D_S}[\ell(\bx, \by;\theta^*)L_{S,T}]\le (1+\delta) \ex_{(x,y)\sim D_S}[\ell(\bx, \by;\theta^*)]\\&=(1+\delta)e_S(\theta^*),
\end{align*}
with the inequality follows from Assumption~\ref{asm:likelihoodRatio}. 
Therefore, we have the desired result.
\end{proof}

\subsection{Effects of Inter-Product Ordering Constraints}
\label{sec:constraintTH}

In discrete choice-based price optimization, a firm offers $K$ substitutable products to heterogeneous customers who select the utility-maximizing option (including a no-purchase alternative). In many settings---airline fare classes, hotel room tiers, subscription plans---business logic requires a price hierarchy among products:
\begin{equation}\label{eq:constraint}
    p_{\pi(k)} \geq p_{\pi(k-1)} + \Delta_k, \quad k = 2, \ldots, K,
\end{equation}
where $\pi$ is a known ordering (e.g., economy $\leq$ business $\leq$ first class) and $\Delta_k \geq 0$ are minimum inter-tier price gaps. When prices are output directly by a neural network, these constraints must be enforced structurally.

\begin{proposition}[Ordering constraints by construction]\label{prop:ordering}
Let $z_1, \ldots, z_K \in \mathbb{R}$ be unconstrained network outputs. Define prices recursively:
\begin{align}
    p_{\pi(1)} &= z_1, \label{eq:base}\\
    p_{\pi(k)} &= p_{\pi(k-1)} + \Delta_k + \operatorname{softplus}(z_k), \quad k = 2, \ldots, K, \label{eq:recursion}
\end{align}
where $\operatorname{softplus}(x) = \log(1 + e^x)$. Then $p_{\pi(k)} > p_{\pi(k-1)} + \Delta_k$ for all $k$ and all $z \in \mathbb{R}^K$.
\end{proposition}

\begin{remark}[Reasoning and implications]
From positivity of softplus
%: since $e^x > 0$ for all $x \in \mathbb{R}$, we have $\operatorname{softplus}(x) = \log(1+e^x) > 0$, and thus 
we have, $p_{\pi(k)} - p_{\pi(k-1)} = \Delta_k + \operatorname{softplus}(z_k) > \Delta_k$. 
%The reparameterization is smooth ($C^\infty$), bijective onto the strictly feasible interior, and gradient-friendly: $\partial p_{\pi(k)}/\partial z_k = \sigma(z_k) \in (0,1)$. 
This construction is a standard technique for enforcing positivity in neural network outputs
%~\cite{kingma2014auto, huang2018neural}, 
applied here to the multi-product pricing setting.
\end{remark}

\begin{remark}[Limitations]\label{rem:limitations}
While the softplus reparameterization cleanly enforces ordering, it might not be sufficient in the presence of  additional constraints. For example, price range constraints in the forms of $\ell_k \leq p_k \leq u_k, k=1, \ldots K$ will become difficult to enforce given the recursive expansion~\eqref{eq:recursion}. Furthermore, average price constraints $\frac{1}{K}\sum_k p_k \geq \bar{p}$ become non-separable and non-linear in $z$. Moreover, when target price gaps are near-tight ($p_{\pi(k)} - p_{\pi(k-1)} \approx \Delta_k$), the required inverse $z_k = \log(e^{p_{\pi(k)} - p_{\pi(k-1)} - \Delta_k} - 1)$ has the tendency to go to $-\infty$, thus creating numerical instability during training. These limitations motivate a dedicated constraint satisfaction layer that operates directly in price space, jointly enforcing bounds, ordering, and aggregate constraints without distorting the network's output representation.
\end{remark}

\newpage
\section{Omitted Details for Experiments}

\subsection{Datasets for Discrete Choice Modeling}
\label{appendix:dcm_datasets}
Several publicly available datasets have been adapted for discrete choice research. In retail and e-commerce, the Amazon product data \citep{ni2019amazon} and Instacart online grocery dataset \citep{instacart2017} provide rich choice signals from millions of user interactions that can be restructured as discrete choice problems. For transportation choices, the Boston Uber/Lyft dataset \citep{boston_rideshare2019} offers fare quotes with contextual features, while the Swissmetro dataset \citep{bierlaire2001swissmetro} remains a benchmark for revealed and stated preference modeling in transportation research. In addition to the simulated datasets, our proposed approach aims to leverage these additional resources to fine-tune the network model in order to generalize better across different choice contexts.

Details for the three groups of data listed in Table~\ref{tab:revenue-combined}:
\begin{itemize}
\item
\emph{Group~A (simulated, oracle-optimal reference)} comprises Amazon DVD
\citep{farias2013nonparametric}, Swiss Metro \citep{bierlaire2001swissmetro},
and Yogurt \citep{jain1994random}. For each evaluation instance the reference
price is the revenue-maximizing price computed under the known
ground-truth model, so recommended prices can never exceed it and
\emph{regret from optimal} is well defined.
\item
\emph{Group~B (simulated, offered-price reference)} comprises the Bodea Hotel
dataset \citep{newman2014estimation}, which follows the Newman two-step
estimation protocol: the reference prices are simulated offered prices
sampled from the historical Bodea price ranges rather than optimal prices.
Consequently methods can exceed the reference, and the meaningful
quantity is \emph{revenue lift over the offered prices}.
\item
\emph{Group~C (live, no ground truth)} comprises the two datasets drawn from
actual client engagements---B2B tender pricing and airline ancillary pricing
(AIRLINE). Here the true choice model is unknown, so no ground-truth revenue can be computed; we can only compare recommended price movements
against the offered legacy prices and their observed win/loss outcomes.
For these we report the model-free directional metrics PDR and PIR.
\end{itemize}

\subsection{Detailed PDR-PIR Statistics}
\label{sec:detailed_pdr_pir_stats}
% Table (b): All datasets
\begin{table}[thbp]
\centering
\small
\setlength{\tabcolsep}{5pt}
\begin{tabular}{ll cc cc cc}
\toprule
& & \multicolumn{2}{c}{\textbf{PDR} $\uparrow$}
& \multicolumn{2}{c}{\textbf{PIR} $\uparrow$}
& \multicolumn{2}{c}{$\mathbf{\min(\text{PDR},\text{PIR})}$ $\uparrow$} \\
\cmidrule(lr){3-4} \cmidrule(lr){5-6} \cmidrule(lr){7-8}
\textbf{Group} & \textbf{Dataset}
& \textbf{C3PO} & \textbf{MNL}
& \textbf{C3PO} & \textbf{MNL}
& \textbf{C3PO} & \textbf{MNL} \\
\midrule
\multirow{2}{*}{\shortstack[l]{Live\\data}}
 & B2B          & \textbf{0.611} & 0.443 & \textbf{0.673} & 0.299 & \textbf{0.611} & 0.299 \\
 & AIRLINE      & \textbf{0.507} & 0.135 & \textbf{0.476} & 0.318 & \textbf{0.476} & 0.135 \\
\midrule
\multirow{4}{*}{\shortstack[l]{Simulated\\data}}
 & DVD          & \textbf{0.430} & 0.022 & 0.530          & \textbf{0.955} & \textbf{0.430} & 0.022 \\
 & SWISS  & \textbf{0.640} & 0.331 & 0.410          & \textbf{0.691} & \textbf{0.410} & 0.331 \\
 & YOGURT       & \textbf{0.500} & 0.060 & 0.543          & \textbf{0.956} & \textbf{0.500} & 0.060 \\
 & HOTEL        & 0.410          & ---   & 0.530          & \textbf{1.000} & 0.410          & ---   \\
\midrule
\multicolumn{2}{l}{\textit{Avg (excl.\ Hotel)}}
 & \textbf{0.543} & 0.198
 & 0.522          & \textbf{0.644}
 & \textbf{0.490} & 0.169 \\
\bottomrule
\end{tabular}
\caption{Directional pricing metrics across all domains.
C3PO consistently achieves higher PDR (price-decrease recall among losses)
while MNL tends toward high PIR by uniformly raising prices.
The balanced metric $\min(\text{PDR},\text{PIR})$ favors C3PO in every
dataset where MNL is available, with an average improvement of
$+32.1$\,pp.}
\label{tab:pdrpir-all}
\end{table}

Table~\ref{tab:pdrpir-all} reports PDR (price-decrease recall among losses)
and PIR (price-increase recall among wins) for C3PO and calibrated MNL
across all six domains.

\textbf{MNL scores high on PIR but very low on PDR}, while C3PO maintains
more balanced directional accuracy. This asymmetry reveals that MNL
systematically recommends price \emph{increases} across nearly all
products and instances. The balanced metric $\min(\text{PDR},\text{PIR})$
consequently favors C3PO in every domain where MNL is available, by an
average of $+32.1$\,pp.

\paragraph{MNL overpricing:}
The MNL baseline is calibrated from the same in-context learning (ICL)
data made available to C3PO. With limited data and no access to
the true heterogeneous preference distribution, a calibrated MNL model
 \emph{underestimates} price sensitivity. When $|\beta|$
is biased downward, the resultant markup $1/|\beta|$ is
inflated, shifting the entire optimal price vector upward. 
%The result
%is a method that almost always raises prices relative to the offered
%price, producing high PIR but failing to identify instances where
%prices should be \emph{lowered}---precisely the cases captured by %PDR.
MNL's underestimation of elasticity from limited or aggregate data is well documented in the discrete choice literature. \citet{train2003discrete} notes that the standard MNL model imposes homogeneous preferences across individuals. When consumers differ in their price sensitivity, the estimated coefficient reflects an average response across heterogeneous segments, potentially muting predicted demand responses and elasticities. Similarly, \citet{mcfadden2000mixed} show that accounting for random taste heterogeneity through mixed logit models can materially affect inference on preferences and willingness-to-pay relative to homogeneous specifications.

\subsection{Consolidated Revenue and PDR-PIR Table}
\label{sec:rev-pdr-correlation}

\paragraph{Revenue gain over reference.}
Table~\ref{tab:rev-gain} reports for each dataset in groups A and B, C3PO's directional balance
metric $\min(\text{PDR},\text{PIR})$ alongside the percentage regret
(Group~A) or lift (Group~B) relative to the reference price.

\begin{table}[thbp]
\centering
\small
\setlength{\tabcolsep}{5pt}
\begin{tabular}{l ccc cc}
\toprule
& & & & \multicolumn{2}{c}{\textbf{Regret / Lift}} \\
\cmidrule(lr){5-6}
\textbf{Dataset}
& \textbf{PDR} & \textbf{PIR}
& $\boldsymbol{\min}$\textbf{(PDR,PIR)}
& \textbf{C3PO} & \textbf{MNL} \\
\midrule
\multicolumn{6}{l}{\emph{Group~A --- regret from optimal}}\\
DVD         & 0.43 & 0.53 & \textbf{0.43} & $-0.7\%$  & $-8.6\%$  \\
SWISS       & 0.64 & 0.41 & 0.41          & $-17.4\%$ & $-18.8\%$ \\
YOGURT      & 0.50 & 0.54 & 0.50          & $-17.6\%$ & $-21.4\%$ \\
\midrule
\multicolumn{6}{l}{\emph{Group~B --- lift over offered}}\\
HOTEL       & 0.41 & 0.53 & 0.41          & $+7.0\%$  & $+12.2\%$ \\
\bottomrule
\end{tabular}
\caption{C3PO directional balance and revenue performance. 
C3PO outperforms MNL on regret in all three Group~A datasets, 
with the largest margin on DVD ($7.9$\,pp) and Yogurt ($3.8$\,pp). 
Directional balance $\min(\text{PDR},\text{PIR})$ is above $0.41$ 
on all domains.}
\label{tab:rev-gain}
\end{table}

\subsection{Fine-Tuned Tabular Foundation Model Baseline for B2B Pricing}
\label{appendix::tabpfn_ft}

The TabPFN-FT baseline in Table~\ref{tab:tender-kpi} is obtained by sequentially fine-tuning the pre-trained TabPFN~2.1 model~\citep{hollmann2022tabpfn,hollmann2025accurate} on the simulated discrete-choice datasets described in Section~\ref{sec:simulated_data_gen}.
The model backbone is frozen and only the output head is updated, using MSE loss on the optimal price labels with a learning rate of $10^{-4}$. Fine-tuning uses 3{,}000 of the 12{,}000 training datasets; loss reduction exhibits diminishing returns beyond approximately 2{,}000 datasets. Training completes in under one hour on a single NVIDIA A100 GPU.

TabPFN-FT treats pricing as a regression (imitation-learning) task: given contextual features and 1{,}600 in-context examples, it predicts the optimal price directly. Unlike C3PO, TabPFN-FT does not model the revenue--price curve, ignores pricing constraints, and does not incorporate elasticity priors---the three extensions that distinguish C3PO.

On the B2B tender dataset, TabPFN-FT achieves $\text{PDR}=0.723$ and $\text{PIR}=0.589$, improving substantially over the off-the-shelf TabPFN ($0.767/0.393$) and the calibrated MNL ($0.443/0.299$).  However, its directional balance $\min(\text{PDR},\text{PIR})=0.589$ remains below C3PO's $0.611$, reflecting the limitations of imitation-only learning: without an explicit revenue objective, TabPFN-FT tends to under-price (high PDR, lower PIR), a bias that C3PO's revenue head and elasticity prior correct.  This comparison motivates the triple-head architecture and confirms that fine-tuning alone, while convenient and beneficial, is insufficient for balanced pricing recommendations.

\subsection{Synthetic Data Generation (SDG)}
\label{appendix:Synth_Data_Gen}
In addition to simulated data, C3PO can be trained, fine-tuned, and evaluated using several approaches for generating synthetic tabular data. The TabPFN-based generation follows the approach of \citep{ma2023tabpfgen} of using a pretrained transformer to produce class‑logit functions 
$\boldsymbol{\mu}(\bx)$ whose softmax defines $\Pr(y \mid \bx)$.

Beyond TabPFN, we also tested Synthetic Data Vault (SDV) methods that include Gaussian Copulas, CTGAN, TAVE, and bootstrap-based generators.
While these methods are vital for finetuning C3PO with synthesized choice structures from different domains, they do not come with analytically derivable decision optimality conditions. Therefore, we restricted our focus to simulated data from discrete choice models (e.g., MNL), which offer unlimited samples and known optimality conditions for benchmarking, albeit under idealized modeling assumptions.

\subsection{Sample Code to Generate Recommended Price Labels}
\label{appendix::label_gen}

This is a plain-vanilla python implementation of a known method based on ~\cite{gallego2014multiproduct} to enable reproducibility of the controlled experiment with NL and MNL models. Similar algorithms can be used to generate good quality labels for other models such as MMNL, Iso-elastic, and Linear formulations. 
\begin{lstlisting}[caption={Python implementation of Multinomial and Nested Logit pricing}, label={lst:nl-opt-MNL}]
import numpy as np
import math
from typing import Dict, Any

def calculate_prices_nested_logit(alpha: np.ndarray,
                                beta: np.ndarray,
                                nest_assignments: np.ndarray,
                                lam: float,
                                tau_nest: Dict[int, float] = None,
                                tol: float = 1e-10,
                                maxit: int = 10000) -> Dict[str, Any]:
    
    # Generating price recommendations under Multinomial and Nested Logit Models (Based on Gallego-Wang 2014).
    # Optimal for MNL and a fast heuristic for NL
    
    alpha = np.asarray(alpha, dtype=float)
    beta = np.asarray(beta, dtype=float)
    nest_assignments = np.asarray(nest_assignments, dtype=int)
    if np.any(beta >= 0):
        raise ValueError("All beta_i must be negative.")

    N = len(alpha)
    K = len(np.unique(nest_assignments))

    alpha_eff = alpha.copy()
    if tau_nest is not None:
        for m, tau in tau_nest.items():
            idx = np.where(nest_assignments == m)[0]
            alpha_eff[idx] = alpha_eff[idx] * (1.0 + tau)

    def nested_logit_probs_given_prices(p: np.ndarray):
        u = alpha_eff + beta * p
        S = {}
        P_cond = np.zeros(N, dtype=float)
        for m in range(K):
            idx = np.where(nest_assignments == m)[0]
            exp_scaled = np.exp(u[idx] / lam)
            S[m] = exp_scaled.sum()
            P_cond[idx] = exp_scaled / S[m]
        denom = sum(S[m] ** lam for m in range(K)) + 1.0
        P_nest = {m: (S[m] ** lam) / denom for m in range(K)}
        q_outside = 1.0 / denom
        q_inside = np.zeros(N, dtype=float)
        for m in range(K):
            idx = np.where(nest_assignments == m)[0]
            q_inside[idx] = P_nest[m] * P_cond[idx]
        return q_inside, q_outside

    def expected_rev_given_R(R: float):
        p = R - 1.0 / beta
        q_inside, q_outside = nested_logit_probs_given_prices(p)
        ER = float(np.sum(p * q_inside))
        return ER, p, q_inside, q_outside

    R = float(max(1e-6, np.mean(1.0 / (-beta))))
    damping = 0.5
    for _ in range(maxit):
        ER, p, q, q0 = expected_rev_given_R(R)
        R_next = (1.0 - damping) * R + damping * ER
        if abs(R_next - R) <= tol * max(1.0, abs(R)):
            R = R_next
            break
        R = R_next
    else:
        def F(x):
            v, *_ = expected_rev_given_R(x)
            return v - x
        a, bnd = 0.0, max(10.0, R*4.0 + 10.0)
        fa, fb = F(a), F(bnd)
        tries = 0
        while fa*fb > 0 and tries < 60:
            bnd = bnd*2.0 + 10.0
            fb = F(bnd)
            tries += 1
        if fa*fb > 0:
            raise RuntimeError("Failed to bracket R* in NL case.")
        for _ in range(200):
            mid = 0.5*(a+bnd); fm = F(mid)
            if abs(fm) < 1e-12 or abs(bnd - a) < 1e-12:
                R = mid; break
            if fa*fm < 0:
                bnd, fb = mid, fm
            else:
                a, fa = mid, fm

    ER, p, q, q0 = expected_rev_given_R(R)
    return {"price": p, "revenue": ER, "q_in": q, "q_out": q0, "method": "fixed_point"}
\end{lstlisting}

\subsection{Elasticity Prior from the Literature}
\label{appendix::label_elas}

\begin{table}[h]
\centering
\caption{Sample of representative price elasticity of demand values.  See \cite{babic2005microeconomics} and \cite{mankiw1998principles} for details.}
\label{tab:elasticities-pr}
\begin{tabular}{l c}
\hline
\textbf{Product / Category} & \textbf{Elasticity} \\
\hline
Water (basic consumption)   & $-0.1$ \\
Electricity (short-run)     & $-0.1$ to $-0.3$ \\
Gasoline (short-run)        & $-0.2$ to $-0.3$ \\
Milk                        & $-0.5$ \\
Gasoline (long-run)         & $-0.6$ to $-0.8$ \\
Electricity (long-run)      & $-0.7$ \\
Restaurant meals            & $-2.3$ \\
Air travel (long-run)       & $-2.0$ \\
Luxury goods                & $-1.5$ to $-3.0$ \\
Automobiles (long-run)      & $-1.2$ to $-1.5$ \\

\hline
\end{tabular}
\end{table}

\begin{minipage}{\textwidth}
\begin{lstlisting}[caption={Python example of a prompt to obtain an elasticity prior range}, label={lst:nl-opt}]

    prompt = (
        f"<|system|>You are an economics professor who teaches from Mankiw's "
        f"Principles of Economics and Pindyck & Rubinfeld's Microeconomics.</s>"
        f"<|user|>Price elasticity of demand (e) is negative for normal goods. "
        f"INELASTIC means |e| < 1 (between -1.0 and 0): gasoline $e\sim -0.3$, "
        f"bread $e\sim -0.2$, milk $e\sim -0.5$. "
        f"ELASTIC means |e| > 1 (more negative than -1.0): restaurant meals "
        f"$e\sim -1.6$, vacations $e\sim -2.5$, new cars $e\sim -1.5$.\n\n"
        f"Is '{product_string}' typically elastic or inelastic based on "
        f"economics textbooks? "
        f"Answer with exactly one word: 'elastic' or 'inelastic'.</s>"
        f"<|assistant|>")
    
\end{lstlisting}
\end{minipage}

\subsection{Ablation Studies Tables}
This section reports detailed dataset-level experimental results that are summarized in the main section. Across these datasets, we test the model on more than 600,000 simulated price points, and a cumulative total of over 300,000 sample price points drawn from real-world and public datasets. The underlying discrete choice models are diverse and include logit, neural network, and tree-ensemble based choice structures, some of which are not seen during training. 

\label{appendix::label_ablation}

\begin{table}[ht]
	\centering
	\small
	\caption{C3PO model performance. PDR/PIR shown as Predicted/Optimal.}
	\label{tab:baseline}
	\begin{tabular}{lcccc}
		\toprule
		Data set & MAE & PDR & PIR\\
		\midrule
		B2B & 0.2723 & 0.61/0.51 & 0.67/0.41 \\
		SIMULATED & 0.4150 & 0.49/0.48 & 0.43/0.61 \\
		%AIR & 0.2590 & 0.57/0.50 & 0.46/0.45 & 0.00/0.00 \\
		DVD & 0.4742 & 0.43/0.52 & 0.53/0.63 \\
		HOTEL & 0.2603 & 0.41/0.43 & 0.53/0.61 \\
		AIRLINE & 0.3321 & 0.51/0.51 & 0.48/0.50 \\
		SWISS & 0.4632 & 0.64/0.42 & 0.41/0.62 \\
		%ELEC & $\cdot$ &  $\cdot$/$\cdot$ & $\cdot$/$\cdot$ & $\cdot$/$\cdot$ \\
		YOGURT & 0.3491 & 0.50/0.53 & 0.54/0.52 \\
		\bottomrule
	\end{tabular}
\end{table}
\begin{table}[ht]
	\centering
	\small
	\begin{tabular}{lcccc}
		\toprule
		Data set & K & Win Rate & Prior Elasticity & Choice Model \\
		\midrule
		SIMULATED & 6 & 0.2351 & (-3.00, -0.50) & MULTIPLE \\
		B2B & 1 & 0.3419 & (-2.00, -1.00) & UNKNOWN \\
		DVD & 15 & 0.4442 & (-1.50, -0.50) & MMNL \\
		HOTEL & 6 & 0.9086 & (-2.50, -1.50) & MNL \\
		AIRLINE & 2 & 0.1378 & (-2.50, -0.50) & UNKNOWN \\
		SWISS & 3 & 0.7508 & (-0.60, -0.20) & TASTENET-MNL \\
		YOGURT & 4 & 0.9639 & (-1.50, -0.50) & MMNL \\
		\bottomrule\\
	\end{tabular}
	\caption{Dataset characteristics and metadata.}
	\label{tab:meta}
\end{table}
\begin{table}[h]
	\centering
	\small
	\caption{IMITATION-ONLY ablation results. PDR/PIR shown as Scenario/Baseline.}
	\label{tab:imitation_only}
	\begin{tabular}{lcccc}
		\toprule
		Eval & MAE & PDR & PIR\\
		\midrule
		SIMULATED & 0.4211 & 0.54/0.49 & 0.35/0.43 \\
		B2B & 0.3279 & 0.90/0.61 & 0.17/0.67  \\
		%AIR & 0.3315 & 0.83/0.57 & 0.21/0.46 & 0.00/0.00 \\
		DVD & 0.4181 & 0.45/0.43 & 0.54/0.53  \\
		HOTEL & 0.2977 & 0.25/0.41 & 0.70/0.53 \\
		AIRLINE & 0.2383 & 0.49/0.51 & 0.49/0.48  \\
		SWISS & 0.3633 & 0.68/0.64 & 0.37/0.41  \\
		%ELEC &  & / & / & / \\
		YOGURT & 0.3435 & 0.50/0.50 & 0.54/0.54  \\
		\bottomrule
	\end{tabular}
\end{table}
\begin{table}[h]
	\centering
	\small
	\caption{ICL-OFF ablation results. PDR/PIR shown as Scenario/Baseline.}
	\label{tab:icl_off}
	\begin{tabular}{lcccc}
		\toprule
		Eval & MAE & PDR & PIR \\
		\midrule
		SIMULATED & 0.7430 & 0.15/0.49 & 0.87/0.43 \\
		B2B & 0.8009 & 0.07/0.61 & 1.00/0.67  \\
		%AIR & 0.6063 & 0.03/0.57 & 0.97/0.46 & 0.79/0.00 \\
		DVD & 0.8208 & 0.46/0.43 & 0.53/0.53  \\
		HOTEL & 0.4262 & 0.31/0.41 & 0.68/0.53  \\
		AIRLINE & 0.7096 & 0.02/0.51 & 0.96/0.48  \\
		SWISS & 0.5361 & 0.42/0.64 & 0.62/0.41  \\
		%ELEC &  & / & / & / \\
		YOGURT & 0.6423 & 0.13/0.50 & 0.87/0.54 \\
		\bottomrule
	\end{tabular}
\end{table}
\begin{table}[h]
	\centering
	\small
	\caption{SIMPLE-ICL ablation results. PDR/PIR shown as Scenario/Baseline.}
	\label{tab:simple_icl}
	\begin{tabular}{lcccc}
		\toprule
		Eval & MAE & PDR & PIR \\
		\midrule
		SIMULATED & 0.4135 & 0.51/0.49 & 0.41/0.43 \\
		B2B & 0.3363 & 0.31/0.61 & 0.88/0.67  \\
		%AIR & 0.3259 & 0.51/0.57 & 0.51/0.46 & 0.06/0.00 \\
		DVD & 0.4115 & 0.50/0.43 & 0.45/0.53  \\
		HOTEL & 0.3316 & 0.36/0.41 & 0.60/0.53 \\
		AIRLINE & 0.3552 & 0.49/0.51 & 0.53/0.48  \\
		SWISS & 0.4690 & 0.63/0.64 & 0.42/0.41 \\
		%ELEC &  & / & / & / \\
		YOGURT & 0.3609 & 0.40/0.50 & 0.59/0.54  \\
		\bottomrule
	\end{tabular}
\end{table}
\begin{table}[h]
	\centering
	\small
	\caption{PRIOR-OFF ablation results. PDR/PIR shown as Scenario/Baseline.}
	\label{tab:prior_off}
	\begin{tabular}{lcccc}
		\toprule
		Eval & MAE & PDR & PIR \\
		\midrule
		SIMULATED & 0.4186 & 0.48/0.49 & 0.47/0.43 \\
		B2B & 0.3036 & 0.35/0.61 & 0.87/0.67  \\
		%AIR & 0.3272 & 0.52/0.57 & 0.49/0.46 & 0.00/0.00 \\
		DVD & 0.4553 & 0.62/0.43 & 0.32/0.53  \\
		HOTEL & 0.3465 & 0.36/0.41 & 0.61/0.53  \\
		AIRLINE & 0.3796 & 0.44/0.51 & 0.58/0.48  \\
		SWISS & 0.4685 & 0.64/0.64 & 0.41/0.41  \\
		%ELEC & $\cdot$ & $\cdot$/$\cdot$ & $\cdot$/$\cdot$ & $\cdot$/$\cdot$ \\
		YOGURT & 0.3516 & 0.55/0.50 & 0.51/0.54  \\
		\bottomrule
	\end{tabular}
\end{table}
\begin{table}[h]
	\centering
	\small
	\caption{Constraint violations: Baseline (unconstrained predictions).}
	\label{tab:constr_baseline_results}
	\begin{tabular}{ll rrr}
		\toprule
		Eval & Constraint & Abs Mean $\pm$ Std & Abs Max & \% Mean $\pm$ Std \\
		\midrule
		%  AIR & Bound & 0.00 $\pm$ 0.00 & 0.00 & 0.0 $\pm$ 0.0\% \\
		%   & Order & 0.00 $\pm$ 0.00 & 0.00 & 0.0 $\pm$ 0.0\% \\
		%   & AvgPrc & 0.00 $\pm$ 0.00 & 0.00 & 0.0 $\pm$ 0.0\% \\
		\addlinespace
		AIRLINE & Bound & 0.07 $\pm$ 0.02 & 0.22 & 3.3 $\pm$ 1.0\% \\
		& Order & 0.00 $\pm$ 0.00 & 0.00 & 0.0 $\pm$ 0.0\% \\
		& AvgPrc & 0.00 $\pm$ 0.00 & 0.00 & 0.0 $\pm$ 0.0\% \\
		\addlinespace
		B2B & Bound & 0.00 $\pm$ 0.00 & 0.00 & 0.0 $\pm$ 0.0\% \\
		& Order & 0.00 $\pm$ 0.00 & 0.00 & 0.0 $\pm$ 0.0\% \\
		& AvgPrc & 0.25 $\pm$ 0.21 & 0.82 & 17.9 $\pm$ 15.5\% \\
		\addlinespace
		DVD & Bound & 1.21 $\pm$ 0.11 & 1.70 & 60.7 $\pm$ 5.7\% \\
		& Order & 0.35 $\pm$ 0.01 & 1.64 & 45.6 $\pm$ 1.1\% \\
		& AvgPrc & 0.08 $\pm$ 0.01 & 0.13 & 10.7 $\pm$ 1.5\% \\
		\addlinespace
		HOTEL & Bound & 0.00 $\pm$ 0.00 & 0.00 & 0.0 $\pm$ 0.0\% \\
		& Order & 0.19 $\pm$ 0.01 & 1.10 & 15.4 $\pm$ 0.6\% \\
		& AvgPrc & 0.00 $\pm$ 0.00 & 0.00 & 0.0 $\pm$ 0.0\% \\
		\addlinespace
		SIMULATED & Bound & 0.00 $\pm$ 0.00 & 0.16 & 0.0 $\pm$ 0.0\% \\
		& Order & 0.23 $\pm$ 0.01 & 1.10 & 21.9 $\pm$ 1.2\% \\
		& AvgPrc & 0.21 $\pm$ 0.04 & 0.77 & 19.5 $\pm$ 3.7\% \\
		\addlinespace
		SWISS & Bound & 0.02 $\pm$ 0.08 & 0.45 & 1.0 $\pm$ 3.8\% \\
		& Order & 0.44 $\pm$ 0.03 & 1.19 & 52.4 $\pm$ 14.5\% \\
		& AvgPrc & 0.00 $\pm$ 0.00 & 0.02 & 0.0 $\pm$ 0.1\% \\
		\addlinespace
		YOGURT & Bound & 0.00 $\pm$ 0.03 & 0.374 & 0.19 $\pm$ 1.26\% \\
		& Order & 0.01 $\pm$ 0.02 & 0.43 & 1.89 $\pm$ 2.53\% \\
		& AvgPrc & 0.00 $\pm$ 0.00 & 0.00 & 0.0 $\pm$ 0.0\% \\
		\bottomrule
	\end{tabular}
\end{table}

\begin{table}[h]
	\centering
	\small
	\caption{Constraint violations: Constrained predictions.}
	\label{tab:constr_on_final}
	\begin{tabular}{ll rrr}
		\toprule
		Eval & Constraint & Abs Mean $\pm$ Std & Abs Max & \% Mean $\pm$ Std \\
		\midrule
		%  AIR & Bound & 0.00 $\pm$ 0.00 & 0.00 & 0.0 $\pm$ 0.0\% \\
		%   & Order & 0.02 $\pm$ 0.02 & 0.07 & 2.5 $\pm$ 3.4\% \\
		%   & AvgPrc & 0.00 $\pm$ 0.00 & 0.00 & 0.0 $\pm$ 0.0\% \\
		\addlinespace
		AIRLINE & Bound & 0.00 $\pm$ 0.00 & 0.00 & 0.0 $\pm$ 0.0\% \\
		& Order & 0.08 $\pm$ 0.01 & 0.18 & 18.6 $\pm$ 1.3\% \\
		& AvgPrc & 0.00 $\pm$ 0.00 & 0.00 & 0.0 $\pm$ 0.0\% \\
		\addlinespace
		B2B & Bound & 0.00 $\pm$ 0.00 & 0.00 & 0.0 $\pm$ 0.0\% \\
		& Order & 0.00 $\pm$ 0.00 & 0.00 & 0.0 $\pm$ 0.0\% \\
		& AvgPrc & 0.00 $\pm$ 0.00 & 0.00 & 0.0 $\pm$ 0.0\% \\
		\addlinespace
		DVD & Bound & 0.50 $\pm$ 0.13 & 1.03 & 25.0 $\pm$ 6.5\% \\
		& Order & 0.29 $\pm$ 0.01 & 1.56 & 47.9 $\pm$ 1.1\% \\
		& AvgPrc & 0.13 $\pm$ 0.01 & 0.16 & 16.8 $\pm$ 1.2\% \\
		\addlinespace
		HOTEL & Bound & 0.00 $\pm$ 0.00 & 0.00 & 0.0 $\pm$ 0.0\% \\
		& Order & 0.02 $\pm$ 0.00 & 0.10 & 2.3 $\pm$ 0.4\% \\
		& AvgPrc & 0.00 $\pm$ 0.00 & 0.01 & 0.0 $\pm$ 0.0\% \\
		\addlinespace
		SIMULATED & Bound & 0.03 $\pm$ 0.02 & 0.73 & 1.4 $\pm$ 0.9\% \\
		& Order & 0.01 $\pm$ 0.00 & 0.16 & 0.4 $\pm$ 0.2\% \\
		& AvgPrc & 0.06 $\pm$ 0.01 & 0.65 & 5.4 $\pm$ 0.9\% \\
		\addlinespace
		SWISS & Bound & 0.00 $\pm$ 0.00 & 0.00 & 0.0 $\pm$ 0.0\% \\
		& Order & 0.01 $\pm$ 0.01 & 0.12 & 1.1 $\pm$ 2.1\% \\
		& AvgPrc & 0.02 $\pm$ 0.02 & 0.11 & 4.3 $\pm$ 4.7\% \\
		\addlinespace
		YOGURT & Bound & 0.00 $\pm$ 0.02 & 0.36 & 0.16 $\pm$ 1.1\% \\
		& Order & 0.01 $\pm$ 0.02 & 0.43 & 1.87 $\pm$ 2.6\% \\
		& AvgPrc & 0.00 $\pm$ 0.00 & 0.00 & 0.0 $\pm$ 0.0\% \\
		\bottomrule
	\end{tabular}
\end{table}

\subsection{C3PO Training Loss}
\label{appendix::label_training_loss}
For each mini-batch of $B$ customers, a random 30\% are designated as in-context examples whose optimal prices are revealed to the ICL transformer; the loss is computed on the remaining queries $\mathcal{Q}$. The total loss~\eqref{eq:total_loss} comprises five terms.

\paragraph{Price imitation ($\mathcal{L}_{\text{price}}$).}
A blend ($\alpha{=}0.9$) of symmetric Huber loss and an asymmetric variant that estimates per-customer elasticity from context by comparing average revenue at below-median vs.\ above-median price points. The resulting ratio $\rho_i \in [0.5, 2.0]$ up-weights over-pricing errors for elastic customers and under-pricing errors for inelastic ones.

\paragraph{Revenue calibration ($\mathcal{L}_{\text{rev}}$).}
Huber regression of predicted vs.\ actual revenue at all $P$ context price points.

%Combines Huber regression of predicted vs.\ actual revenue at all $P$ context price points (weight 0.7) with a margin-based ranking loss (weight 0.3) ensuring $\hat{r}(\mathbf{p}^*) > \hat{r}(\mathbf{p}_{\text{rand}}) + \delta$ with $\delta{=}0.01$.

\paragraph{Revenue maximization ($\mathcal{L}_{\text{revmax}}$).}
Predicted query prices are passed through the \emph{frozen} revenue head and the negative 'ground truth' revenue is minimized: $\mathcal{L}_{\text{revmax}} = -\mathbb{E}_{i \in \mathcal{Q}}[\hat{r}(\hat{\mathbf{p}}_i) / \bar{r}_{\text{ctx}}]$. The weight $\lambda_{\text{revmax}}(t)$ is zero during a warmup ($t \leq 200$), ramps linearly to $\lambda_{\text{target}}{=}0.5$ over the next 200 iterations.
%, and is adaptively dampened when running-average revenue calibration error is high.

\paragraph{Constraint penalty ($\mathcal{L}_{\text{constr}}$).}
Penalizes violations of per-product bounds, inter-product ordering, and a global weighted-average price floor on pre-projection prices.
%, encouraging the network to internalize constraints rather than relying solely on the differentiable projection layer.

\paragraph{Auxiliary regularizers ($\mathcal{L}_{\text{reg}}$).}

Two terms provide economic inductive bias: (i)~an \emph{elasticity prior alignment} loss that penalizes the model's inferred elasticity $\hat{\varepsilon}$ for falling outside a user-specified range $[\varepsilon_{\text{lo}}, \varepsilon_{\text{hi}}]$ and discourages price adjustments opposing the elasticity mismatch direction, and (ii)~an anchor alignment loss motivated by the Lerner condition (\cite{anderson1992discrete}) that nudges predictions toward analytically derived optimal-price anchors (from the first-order condition $\varepsilon^*{=}{-}1$).
%, weighted by a learned confidence score.
Appendix section \ref{appendix::label_elas_basics} presents a detailed elasticity analysis.

Default loss weights are: $w_{\text{price}}{=}1.0$, $\lambda_{\text{rev}}{=}0.25$, $\lambda_{\text{revmax}}{=}0{\to}0.5$ (ramped), $\lambda_{\text{constr}}{=}2.0$, $\lambda_{\text{prior}}{=}0.75$, $\lambda_{\text{anchor}}{=}0.25$.

\iffalse
\subsection{Summary of Weights}

Table~\ref{tab:loss_weights} summarizes the default loss weights used in all experiments.

\begin{table}[h]
\centering
\caption{Default loss weights.}
\label{tab:loss_weights}
\begin{tabular}{lcc}
\toprule
\textbf{Term} & \textbf{Symbol} & \textbf{Weight} \\
\midrule
Price imitation & $w_{\text{price}}$ & 1.0 \\
Revenue calibration & $\lambda_{\text{rev}}$ & 0.25 \\
Revenue maximization & $\lambda_{\text{revmax}}$ & $0 \to 0.5$ (ramped) \\
Constraint penalty & $\lambda_{\text{constr}}$ & 2.0 \\
Elasticity prior & $\lambda_{\text{prior}}$ & 0.75 \\
Anchor alignment & $\lambda_{\text{anchor}}$ & 0.25 \\
\bottomrule
\end{tabular}
\end{table}
\fi

\subsection{Training Data Generation and Loss Plots}
\label{appendix::label_training_data_and_plots}

The B2B experiment (B2B) represents a fundamental discrete choice model. 
%Its experimental details are discussed in the next section \ref{appendix::b2b_details}. 
The airline experiment (AIRLINE) relies on anonymized booking data derived from one of the largest global legacy airline carriers. %The ICL examples are generated from short, medium, and long haul markets, while the predictions are on nine other test markets (three each of Short/Medium/Long haul markets). 
The DVD and HOTEL datasets are derived from the discrete-choice literature. The DVD dataset in particular, tests the model's scalability in terms of generating prices for a much larger set of choices than what it was trained on. The last two datasets test C3PO performance in a low-elasticity setting. For each domain, we provide 1,600 ICL examples with elasticity priors drawn from the literature. 
%TODO- shiva- check A detailed description of generating the price labels for the datasets from the literature is given in the appendix.

The model is trained sequentially on each dataset using the aforementioned loss function with a learning rate of $10^{-4}$, AdamW optimizer, smooth-L1 and ReLU loss functions, $d_{\text{model}}=32$, a batch size of 128, and $30\%$ of the rows in each batch are used as examples to enable in-context learning. 
%The weights for the key training loss and regularization terms can be tuned through hyper-parameter optimization (HPO), but we chose these simple settings: price: 1.0, revenue: 0.25, elasticity: 0.25, anchor: 0.25, elasticity prior: $0.75$, and the weight for constraint: 2.0.
The resultant number of trainable parameters is about 337,000.

%TODO give network param sizes in appendix 
The overall training time is less than an hour on a Windows Intel 20-core laptop with a NVIDIA RTX4000 GPU and 64GB RAM. 

  Customer heterogeneity is simulated by incorporating seven binary customer attributes for each discrete-choice dataset, which yields $2^7$ different customer segments per dataset. The choice cardinality varied between $[2, 6]$. The average price elasticity of the simulated product choices range between (-3.0, -1.0), reflecting mostly price-elastic settings. Note the choices with elasticity less than $-1.0$ are classified as price elastic, whereas elasticity values in the interval $(-1, 0)$ indicate weakly elastic products~\cite{samuelsoneconomics}. Given this training data choice, C3PO is suited for price-elastic purchase settings where the optimal prices typically lie in the interior of the bounds and do not simply drift to the upper bound, which is often an optimal choice in a weakly elastic setting.

For training, we additionally incorporate information about the feasible space of the underlying stochastic optimization problem through a flattened matrix of counterfactual ``what-if'' outcomes generated from multiple feasible price vectors. These outcomes are produced using the underlying choice models in our controlled experiments. %discrete choice model. (MNL or NL in our controlled experiments, and the best-fitting model for real-world data).
Specifically, for each data row, we generate 50 random price vectors per segment and record the corresponding expected revenues. We generate one such row per segment, yielding a total of $128$ rows per dataset. Each row contains the recommended price vectors corresponding to the given discrete-choice parameter settings. As a result, the contextual features and recommendation for each row take the form $(\boldsymbol{z},\, 50 \times (\boldsymbol{p}, r^*),\, \boldsymbol{y})$, with $L + 50 \times (K + 1) + K$ columns.

We generate prices $\bp$ for the 'what-if' data by sampling from a normal distribution with mean and standard deviation equal to $1.0$, and clip the sampled values to lie within the interval $[0, 2]$. This design restricts prices to a narrow range around the mean, reflecting real-world pricing distributions that typically arise from limited experimentation. Prior to predicting prices in a new domain, we scale the contextual data, the prices and revenues by the maximum value seen in the ICL dataset.

% The labels to be predicted are the optimal actions of the service provider that maximize a seller-specified KPI such as expected revenue or margin in which the customers make their choices in response to the seller's actions on the product offerings. For known choice models such as NL and MNL, we calculate the recommended  prices $\bp_i$ for each customer segment $i$ using a practically reliable and scalable method based on the optimality criteria described in \cite{gallego2014multiproduct}. We improve reliability by incorporating additional fallback logic and numerical safeguards. The implemented python code is available in Appendix~\ref{appendix::label_gen}.

%The algorithm finds the optimal prices for MNL and also serves as a fast high-quality heuristic for the nested logit extension using a fixed‑point iteration on expected revenue. Given an expected revenue value, we calculate the corresponding prices, choice probabilities, and re-evaluate the expected revenue. We iterate until we converge or fall back to a bisection search otherwise. 
\newpage
We now describe how to generate recommended prices for in-context learning and training. Note that the prediction targets are the optimal actions of the product/service provider that maximize a seller's expected revenue or profit margin, under the assumption that customers respond to the seller’s actions through their discrete choice behavior over the offered products. For known choice models, such as the nested logit (NL) and multinomial logit (MNL) models, we compute the recommended prices {$\bp_i=\bp^{\circ}(\bx_n,\bz_i)$} for each customer segment $i$ using the procedure based on the optimality conditions described in~\cite{gallego2014multiproduct}. For robustness, we incorporate additional fallback logic and numerical safeguards. The Python implementation is provided in Appendix~\ref{appendix::label_gen}. 
%For models such as MMNL, Linear, and Iso-elastic, we employ standard nonlinear optimization search methods.

\begin{figure}[h]
    \centering
    \includegraphics[width=0.7\linewidth]{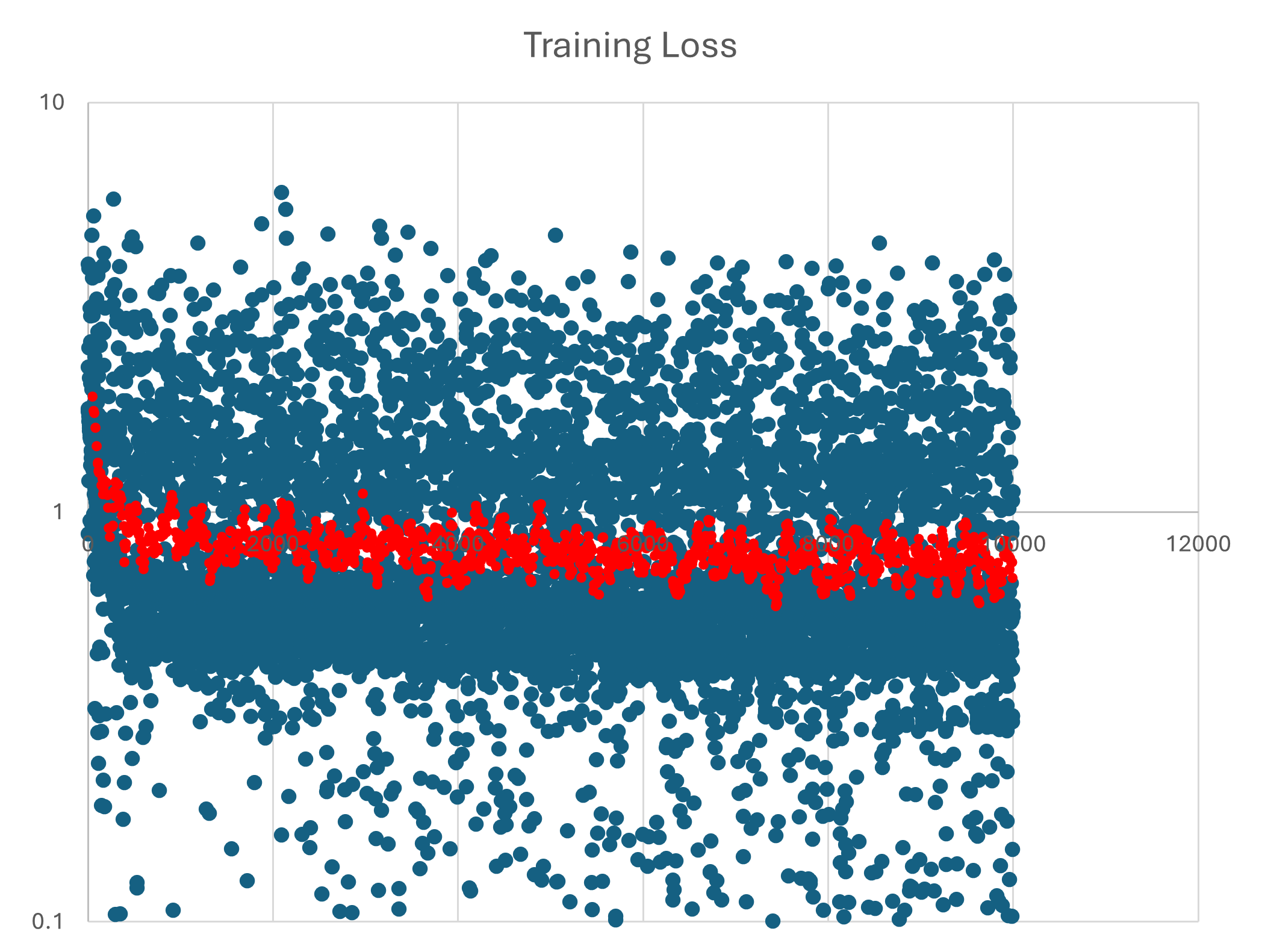}
    \caption{Training loss as a function of dataset count, including a log-scaled y-axis and a 50-point moving average.}
    \label{fig:loss-vs-dataset}
\end{figure}

\begin{figure}[h]
    \centering
    \includegraphics[width=0.7\linewidth]{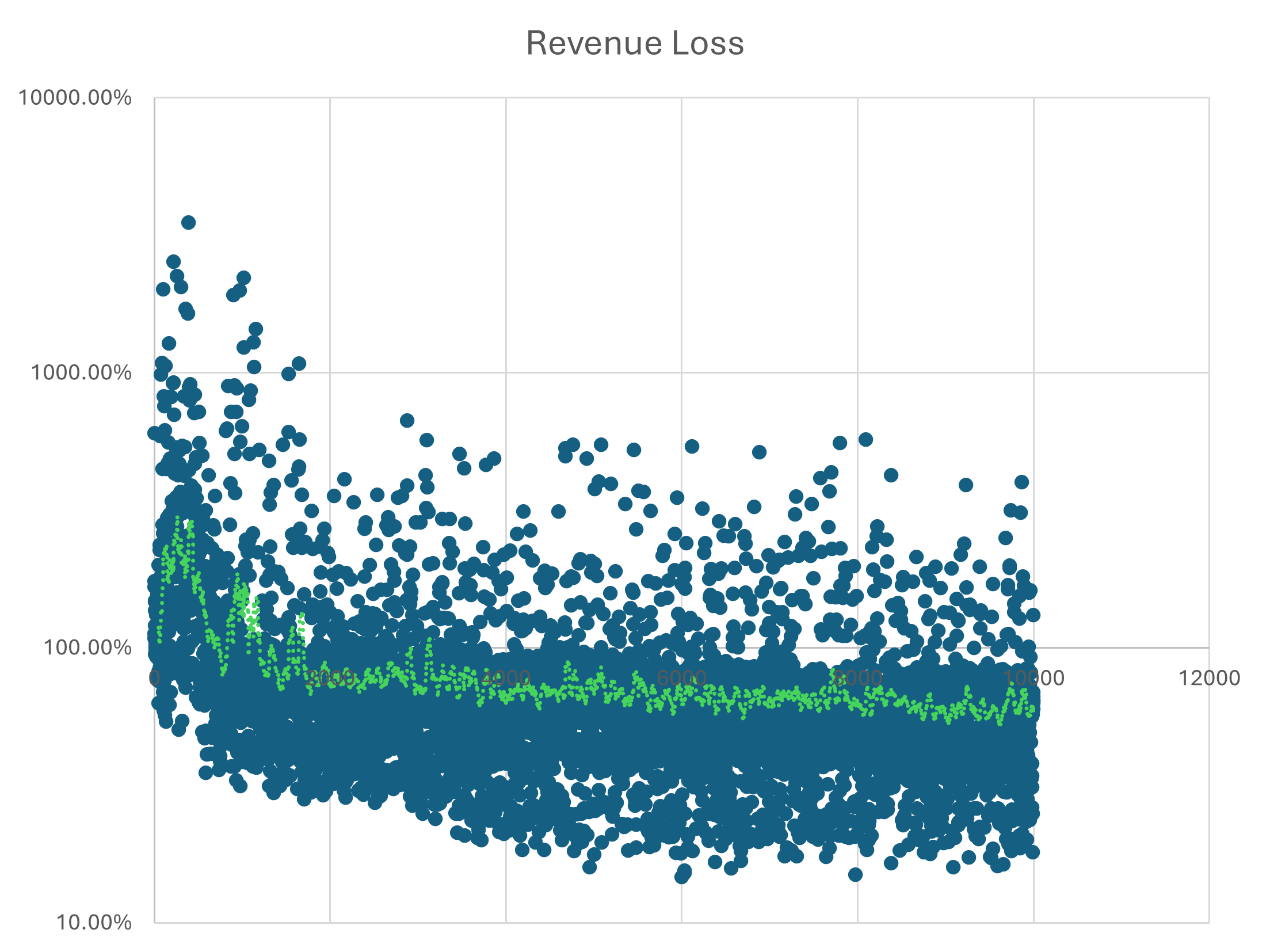}
    \caption{Revenue-loss term as a function of dataset count, including a log-scaled y-axis and a 50-point moving average.}
    \label{fig:revloss-vs-dataset}
\end{figure}

\begin{figure}[h]
    \centering
    \includegraphics[width=0.8\linewidth]{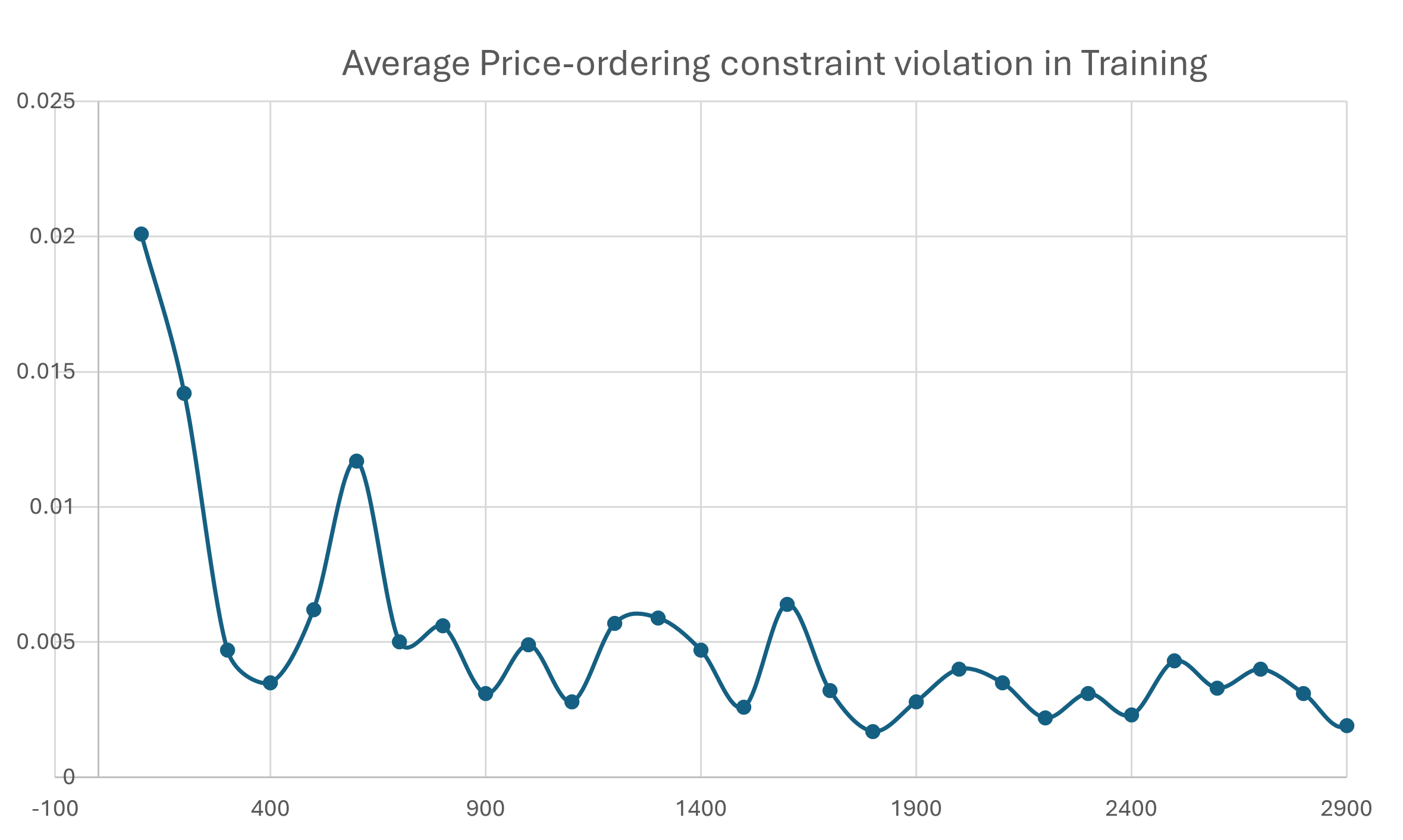}
    \caption{Average price ordering constraint violation as a function of dataset count, including a linear y-axis and a 50-point moving average.}
    \label{fig:constraintloss-vs-dataset}
\end{figure}
\newpage

\iffalse
\begin{figure}[htbp!]
    \centering
    \begin{subfigure}[b]{0.4\textwidth}
        \centering
        \includegraphics[width=\linewidth]{price_vs_competition.pdf}
        \caption{Scatter plot of historical and recommended normalized prices versus predicted competition intensity;}
        \label{fig:price_competition}
    \end{subfigure}
    \qquad
    \begin{subfigure}[b]{0.4\textwidth}
        \centering
        \includegraphics[width=\linewidth]{win_rates2.png}
        \caption{Bar graph comparison of win-rates for legacy and TabPFN-baseline prices by product division.}
        \label{fig:fmb-win-rate}
    \end{subfigure}
\end{figure}
\newpage

\fi
When customer preferences are represented through complex nonlinear interactions or learned via opaque, data-driven classification models for discrete choice, the optimization algorithm described above can only be used as a heuristic procedure. In practical applications, we can adopt the column-generation-based policy optimization approach proposed in~\cite{subramanian2022constrained} to construct near-optimal pricing policies that satisfy a broad range of operational, regulatory, fairness, and business constraints. 

The resulting training dynamics are illustrated in Figure~\ref{fig:loss-vs-dataset}, which reports the total training loss values along with a 50-iteration moving average to highlight the mean convergence trajectory. The reduction in the revenue loss component and the average price-ordering constraint violation (when the constraint layer is active) during training is shown in figures \ref{fig:revloss-vs-dataset} and \ref{fig:constraintloss-vs-dataset}.

\end{document}